\documentclass[11pt]{article}
\usepackage{graphicx}

\usepackage{indentfirst,csquotes}

\usepackage{url}            
\usepackage{booktabs}       
\usepackage{amsfonts, amsmath}       
\usepackage{nicefrac}       
\usepackage{microtype}      
\usepackage{xcolor}         

\usepackage{hyperref}       
\hypersetup{
	colorlinks   = true, 
	urlcolor     = blue, 
	linkcolor    = blue, 
	citecolor   = blue 
}
\usepackage{multirow}
\usepackage{csquotes}

\usepackage[noend]{algpseudocode}
\usepackage{amsthm}
\usepackage{chngcntr}
\usepackage{apptools}
\newtheorem{theorem}{Theorem}
\newtheorem{lemma}[theorem]{Lemma}

\newtheorem{Corollary}{Corollary}[theorem]
\newtheorem*{remark}{Remark}
\newtheorem*{lemma*}{Lemma}
\newtheorem{Proposition}{Proposition}
\newtheorem*{Proposition*}{Proposition}
\newtheorem{assumption}{Assumption}

\usepackage{graphicx}
\usepackage{graphics}
\usepackage{parskip}
\usepackage{amsmath}
\usepackage[ruled,vlined, linesnumbered, algo2e]{algorithm2e}
\usepackage{algorithm}

\usepackage{subfig}

\usepackage{dsfont}
\usepackage{natbib}
\usepackage[leftcaption]{sidecap}

\usepackage[flushleft]{threeparttable}
\usepackage{array,booktabs,makecell}

\usepackage{amsmath,amssymb}
\usepackage[flushleft]{threeparttable}
\usepackage{array,booktabs,makecell}
\usepackage[font=small]{caption}

\usepackage{wrapfig}
\usepackage[normalem]{ulem}
\usepackage{soul}
\usepackage[normalem]{ulem}               
\newcommand\redout{\bgroup\markoverwith
{\textcolor{red}{\rule[0.5ex]{2pt}{0.8pt}}}\ULon}

\makeatletter
\renewcommand{\Indentp}[1]{%
  \advance\leftskip by #1
  \advance\skiptext by -#1
  \advance\skiprule by #1}%
\renewcommand{\Indp}{\algocf@adjustskipindent\Indentp{\algoskipindent}}
\renewcommand{\Indm}{\algocf@adjustskipindent\Indentp{-\algoskipindent}}
\makeatother

\title{Bandit Learning in General Open Multi-agent Systems}

\author{%
  Mengfan Xu\\
  Mechanical and Industrial Engineering \\
  University of Massachusetts Amherst \\
  \texttt{mengfanxu@umass.edu} \\
}

\begin{document}

\maketitle

\begin{abstract}
Open systems, where agents arrive and depart over time, are growing 
in digital platforms. Existing work on online learning in such systems,
however, often relies on structural assumptions that may be restrictive in
practice. In particular, heterogeneous rewards and general agent patterns create three new
challenges: arrivals induce endogenous non-stationarity, agent patterns govern
the rate of information accumulation, and uncertainty brought by new agents
accumulates over time and scales regret. To this end, we formulate a learning problem for open
systems that accommodates both forms of generality. We introduce new concepts to capture the inherent complexities:
the \emph{pre-training degree} of new agents quantifies how much information an
agent carries upon entry, \emph{stability} measures the impact of population
changes on the global optimum, and \emph{global dynamic regret} compares the
cumulative expected reward with that of the varying optimal arms.  We also study a manifold representation of agents and
introduce a \emph{dynamic population embedding} that directly encodes the
population functional for each arm. We develop two UCB-based methods. The first combines entry information with
online observations through agent-wise confidence bounds and aggregates the  estimates and uncertainties for arm selection. The second pools agent-level observations through kernel regression on the manifold and
constructs UCBs by projecting the confidence sets onto the current
population embedding. Entry uncertainty enters the regret linearly through the pre-training error,
while in stable regimes, regret is governed by the time needed to identify the
stable optimal arm. With the manifold as a handle, regret is governed by the
information gained from agent representations. These bounds all depend on
population dynamics, but in different ways. 
\end{abstract}

\section{Introduction}

Multi-armed bandits \citep{auer2002finite, auer2002nonstochastic} have been widely studied for their mathematical beauty and broad applicability. Seminal algorithms couple statistical estimation with uncertainty control \citep{auer2002finite, thompson1933likelihood} and have enabled decision making in e-commerce \citep{xiang2022multi, jiang2021multi, barraza2020introduction, zhu2023scalable} and, more recently, large language models (LLMs) \citep{atalar2026neural, poon2026online, bouneffouf2026multi, bouneffouf2024tutorial}. However, modern systems increasingly involve multiple decision-makers, ranging from distributed services \citep{zhu2021multi, yang2024multi} and store branches \citep{mes2018multi, khatua2011multi} to teams of agents in agentic systems \citep{he2025llm, chen2025optima}. The classical bandit framework no longer directly applies, and  multi-agent multi-armed bandits (MA-MABs) emerge as a recipe  \citep{landgren2016distributed,landgren2016distributed_2,landgren2021distributed,zhu2020distributed,martinez2019decentralized,agarwal2022multi,wangx2022achieving,wangp2020optimal,li2022privacy,sankararaman2019social,chawla2020gossiping,xu2023decentralized,xu2024decentralized,xu2025multi}. In MA-MABs, agents are nodes on a graph. At each round, each agent selects an action, observes the local reward of the chosen arm, communicates with its neighbors, and incorporates the received information into its decisions. The objective is to always select the arm with the highest aggregate reward over all agents. 

To date, it is habitually assumed that the system is closed, in the sense that the set of decision-makers remains fixed throughout. In practice, however, 
data centers and their tenants may expand or shut down
\citep{shehabi2018data,shieh2011sharing}, store branches may open or close
\citep{clapp2014expansions}, and LLM agents may be added to or removed from
agentic workflows \citep{liu2026mas,wang2025agentdropout}. Such systems have an
evolving set of decision-makers and are often described as open multi-agent
systems (OMAS) \citep{deplano2026optimization,liu2026algorithm}. Recent work
has begun to study learning therein
\citep{rosenski2016multi,trinh2021high}, assuming that active
agents are drawn from a fixed pool of bounded size. This may not apply when populations can grow arbitrarily.
\citep{xu2026open} removes this restriction via Poisson processes, assuming
homogeneous rewards across agents. However, both assumptions can still be
restrictive: rewards for the same arm may vary across agents \citep{yang2021cooperative},
while population changes may again be arbitrary \citep{deplano2026optimization}.

Removing them, yet, changes both the objective and the learning process.
As rewards are heterogeneous, an arriving or departing agent changes the
aggregate reward of each arm by a different amount, while population dynamics
determine when and how strongly these changes affect the system. Thus, even when
every agent has stationary rewards, changes in the active population can change
the globally optimal arm. A fixed global benchmark appears inadequate. In the meantime, an arriving agent may bring only limited information about
its own rewards, while information learned from existing agents may not transfer to it under heterogeneity. The less accurate an arriving agent's prior
information is, the more uncertainty it introduces into the learning process.
Population dynamics further determine how much such perturbations accumulate.  

From this interaction, open systems can exhibit different 
intrinsic difficulty. In the worst case, each arrival can introduce enough
uncertainty or shift the aggregate rewards enough to create a new learning
problem. In more favorable regimes, the active population is large enough that later
arrivals have only transient effects and the optimal arm eventually stabilizes. Also favorably and more generally, agents may share latent structures
that allow information to transfer across them, so that what matters for
decision making is easier to learn than their separate rewards.
All in all, the difficulty depends on which aspects of openness are key to the
decision making. Hence, we hope to know:  

\begin{center}
\emph{Can we disentangle the sources of learning difficulty and address them in open systems?}
\end{center}

\paragraph{Main Contributions.} 
We evaluate an arm through its population-level value, obtained by aggregating its rewards over the current population, and target the arm with the largest such value, yielding global dynamic regret. This benchmark exposes two effects of openness. Arrivals introduce uncertainty before observations are collected, which we quantify through the \emph{pre-training error} \(P_m\). Population changes can also alter which arm is globally optimal, motivating our notion of \emph{stability}. 

We develop methods that target different sources of learning difficulty in
open systems. Agent-Wise Global-UCB combines an arriving agent's entry estimate with observations collected after entry. For each agent and arm, it conditionally uses the entry estimate or the empirical estimate according to which has the smaller confidence radius, and then aggregates these conditional estimates and their uncertainties for arm selection. Motivated by the insight that decisions depend only on aggregate rewards, we
represent agents on a common manifold, model each arm by a reward function in
an RKHS, and use a dynamic population embedding to express its aggregate reward
as a population functional. We estimate the reward functions through kernel
ridge regression and project their estimates and confidence sets onto the
dynamic population embedding. In this way, we directly estimate the aggregate
reward and control its uncertainty. We further enable parallel exploration by adapting a top-two sampling
principle. The arm with the largest UCB serves as the leader, while its
plausible competitors form the challenger set; exploratory agents are then
assigned to challengers via matching to maximize the reduction in uncertainty. 

For Agent-Wise Global-UCB, regret decomposes into the cumulative pre-training
error and the post-entry learning cost, which depends on their lifetimes and the statistical rate. These dependencies
are unavoidable on hard instances, up to logarithmic factors. With a persistent 
globally optimal arm, it admits
\(O(\sum_{t=1}^{\tau}M_t),
\)
where \(\tau\) is the burn-in time. Alongside manifold-based learning, we find two complexities that jointly govern regret. The cumulative
information gain
\(\Gamma_T^\Sigma\) measures the information accumulated about the reward
functions, while the population path 
\(
H_T
=
\sum_{t=1}^T
\nicefrac{M_t^2}{\log(1+M_t)}
\)
quantifies the amplification of uncertainty induced by population dynamics. Their interaction yields 
\(
O(
\beta_T\sqrt{\Gamma_T^\Sigma H_T})
\). Under a
Mat\'ern kernel, it yields
\(\widetilde O(
K^{\nu/(2\nu+d)}
T^{1/2+d/(2\nu+d)}
),
\)
determined by the manifold dimension \(d\) and smoothness
\(\nu\). Parallel exploration introduces a  trade-off in the regret bound. It
reduces the amplification due to population size, at the cost of extra
exploration for information acquisition. Whether this improves the regret bound
depends on whether the uncertainty reduction outweighs this cost,
which in turn depends on the matching.

\section{Problem Formulation}
\label{sec:ma-mab}

We study learning in an open system with \(K\) arms,
indexed by \([K]=\{1,2,\ldots,K\}\), over a horizon of \(T\) rounds,
indexed by \(t\in[T]\). We use \(t=0\) to denote the initial system state. At each round \(t\), agents may arrive or depart,
yielding the active set \(\mathcal M_t\). The active agents then select
arms, observe their local rewards, and exchange information via the
communication graph \(G_t=(\mathcal V_t,\mathcal E_t)\), where
\(\mathcal V_t=\mathcal M_t\). Thereafter, we assume complete
communication among active agents to focus on the challenges induced
by reward heterogeneity and population dynamics. For each \(m\in\mathcal M_t\), let
\(\mathcal N_m(t)=\{n\in\mathcal M_t:(m,n)\in\mathcal E_t\}\) denote its
neighbor set.

Agents may arrive and depart according to general patterns herein. We denote \(\mathcal{A}_t=\mathcal{M}_t\setminus\mathcal{M}_{t-1}\) and \(\mathcal{D}_t=\mathcal{M}_{t-1}\setminus\mathcal{M}_{t}\) as the sets of arrivals and departures. 
A common special case is Poisson dynamics
\citep{birnbaum1954statistical,xu2026open}, under which the numbers of arrivals and departures at round \(t\) follow
\(A_t = |\mathcal{A}_t|\sim\mathrm{Poisson}(\lambda_A)\) and
\(D_t = |\mathcal{D}_t|\sim \min\{|M_t|, \mathrm{Poisson}(\lambda_D)\}\), respectively. The population size then
evolves as
\(M_t=|\mathcal{M}_t| = (M_{t-1}+A_t-D_t)_+\),
starting from \(M_0\). Setting
\(\lambda_A=\lambda_D=0\) degenerates to fixed
populations. We only require $M_t \geq 1$ throughout.

If agent \(m\in\mathcal{M}_t\) pulls arm \(i\) at round \(t\), it receives
reward \(0 \leq r_m^i(t) \leq 1\) from a stationary distribution with mean
\(0 \leq \mu_m^i=\mathbb{E}[r_m^i(t)] \leq 1\). Different from the homogeneous setting \citep{xu2026open}, we allow
\(\mu_m^i\neq \mu_n^i\) for \(m\neq n\); the homogeneous model is recovered when
\(\mu_m^i=\mu^i\) for all \(m\) and \(i\). As is standard in the
bandit literature, we assume that the rewards are conditionally and independently $R$-sub-Gaussian given the mean values. To evaluate an arm at the system level, we aggregate its mean rewards over the active agents. 
Precisely, for arm \(i\), let us define the aggregate reward
\(
V_t(i)=\sum_{m\in\mathcal M_t}\mu_m^i,
\)
and its population average
\(
\bar V_t(i)=\frac{1}{M_t}\sum_{m\in\mathcal M_t}\mu_m^i,
\)
where the latter offsets the effect of population size.

\paragraph{Pre-training degree/error.}
A key feature of open systems is that an arriving agent may already carry an
initial estimator that is informative about its  local rewards. For an agent \(m\in\mathcal{A}_t\), and let \(\hat{\mu}_i^m(t)\) denote the
arm-\(i\) estimator available to agent \(m\) upon entry. This estimator may come
from pre-training, side information, or inherited knowledge. We define the
\emph{pre-training degree} and \emph{error} of agent \(m\) by
\(
D_m=\nicefrac{1}{P_m}\) and \(
P_m=\max_{i\in[K]}|
\hat{\mu}_i^m(T_m^a-1)-\mu_m^i|
\); \(D_m=\infty\) when \(P_m=0\). Thus, \(P_m\) measures the worst-case
initialization error of agent \(m\) across arms. \\

\begin{remark}
The notion of \emph{pre-trained agents} in MA-MAB characterizes
the accuracy and flow of prior information in open systems. Inspired by
pre-trained models in the LLM community, it may also provide a new
route for extending MA-MABs to systems of LLM agents. The concept of \emph{pre-training error} is also related to the
information-theoretic view of nonstationary bandits in
\citep{min2023information}, although the source of nonstationarity differs,
which naturally leads to different measures of its complexity. In their setting, the entropy rate of the optimal-action process measures how difficult it is to track a changing environment, while
the information ratio characterizes the cost of acquiring the information
needed to do so. In our setting, nonstationarity arises endogenously as
arrivals change the aggregate reward vector. The 
cumulative pre-training error measures how such uncertainty accumulates over time.
\end{remark}

We then quantify how much the global reward of each arm changes as the population
evolves through
\(
\max_{i\in[K]}|V_t(i)-V_{t-1}(i)|
\)
and
\(
\max_{i\in[K]}|\bar V_t(i)-\bar V_{t-1}(i)|.
\)
Since \(M_t\) is common across arms at round \(t\), \(V_t(i)\) and
\(\bar V_t(i)\) have the same maximizer. Let
\(i_t^\star\in\arg\max_{i\in[K]}\bar V_t(i)\). We track whether the globally optimal arm changes from round \(t-1\) to
round \(t\) through the \emph{stability} indicator
\(
S_t=\mathbf{1}\{i_t^\star=i_{t-1}^\star\}.
\) This notion uniquely belongs to the interplay between openness and reward
heterogeneity. In a closed system, the population is fixed; in
a homogeneous open system, population changes scale all rewards in the same
way. Either way, the identity of $i_t^\star$ is time-invariant and $S_t\equiv 1$.

\textbf{Objective.}
We express system performance in our go-to language of \emph{regret}. Let \(a_m(t)\in[K]\) denote the arm selected by agent \(m\) at round \(t\).  Compared to $i_t^\star$, the instantaneous regret of its choice
\(a_m(t)\) is
\(
\bar V_t(i_t^\star)-\bar V_t(a_m(t)). 
\)
Summing these
losses across all active agents and rounds yields the \emph{global dynamic
regret}
\[
R_T
=
\sum_{t=1}^T
\sum_{m\in\mathcal M_t}
[
\bar V_t(i_t^\star)-\bar V_t(a_m(t))
]
]. 
\]
Here \(R_T\) retains the effect of population size and the \emph{normalized regret} \[\bar R_T = \sum_{t=1}^T
\frac{1}{M_t}
\sum_{m\in\mathcal M_t}
[
\bar V_t(i_t^\star)-\bar V_t(a_m(t))\] removes this scale.  Notably, when all select the same arm \(a_t\), the two quantities simplify
to
\(
R_T
=
\sum_{t=1}^T
[
V_t(i_t^\star)-V_t(a_t)
]\) and \(
\bar R_T
=
\sum_{t=1}^T
[
\bar V_t(i_t^\star)-\bar V_t(a_t)].
\)

\section{Methodology}

We first present a relatively direct method, agent-wise Global-UCB, which learns the mean rewards of individual agents and \emph{controls their uncertainty}, and applies when new agents enter with reward estimates of quantifiable accuracy. We then introduce a manifold-based method that exploits geometric structure across agents and directly learns the aggregate reward of the population and more importantly, \emph{controls its uncertainty.}

\subsection{Agent-wise Learning}\label{subset:3.1}

The first method (Algorithm \ref{alg:general-open-mamab-global}) learns the agent-wise rewards by combining
information available when an agent enters with observations collected after
entry. Let us index distinct agents by their order of first appearance, and let
\(j(m)\) denote the index of agent \(m\). Such an agent carries an entry estimate
\(\hat{\mu}_{m,i}^{\mathrm{entry}}\) with a  pretraining error
\(\bar P_m\) satisfying
\(
\max_{i\in[K]}
\left|
\hat{\mu}_{m,i}^{\mathrm{entry}}-\mu_m^i
\right|
\le \bar P_m.
\)

For \(n\ge1\) and confidence parameter $\delta \in (0,1)$, let us define the empirical confidence radius
\(
c_m(n)
=
\sqrt{
\frac{1}{2n}
\log(
\frac{2Kj(m)(j(m)+1)n(n+1)}{\delta}
)
}
\); $c_m(0)=\infty$. 
For each continuing agent \(m\in\mathcal R_t\) with $\mathcal R_t = \mathcal{M}_t \cap \mathcal{M}_{t-1}$ and arm \(i\), let
\(n_{m,i}(t-1)\) and \(\bar r_{m,i}(t-1)\) denote its number of observations
and empirical mean. We set
\(
\rho_{m,i}(t-1)
=
\min\{\bar P_m,c_m(n_{m,i}(t-1))\},
\)
and use
\(
\hat\mu_{m,i}(t-1)
=
\begin{cases}
\hat\mu_{m,i}^{\mathrm{entry}},
&
\bar P_m\le c_m(n_{m,i}(t-1)),\\
\bar r_{m,i}(t-1),
&
\bar P_m>c_m(n_{m,i}(t-1)).
\end{cases}
\)
Through this, we use the entry estimate until the observations collected yield a tighter confidence radius. 

At round \(t\), the aggregate expected reward
\(V_t(i)=\sum_{m\in\mathcal M_t}\mu_m^i\) is estimated by
\(
\widehat V_t(i)
=
\sum_{m\in\mathcal R_t}\hat\mu_{m,i}(t-1)
+
\sum_{m\in\mathcal A_t}\hat\mu_{m,i}^{\mathrm{entry}}.
\)
We compute the confidence radius for agents in $\mathcal{R}_t$ and $\mathcal{A}_t$ as their vanilla sums, i.e., 
\(
B_t^{\mathrm{stat}}(i)
=
\sum_{m\in\mathcal R_t}\rho_{m,i}(t-1)\) and \(
E_t^A
=
\sum_{m\in\mathcal A_t}\bar P_m.
\)

The total confidence width should naturally add uncertainty from new arrivals,
\(E_t^A\), to that of continuing agents, \(B_t^{\mathrm{stat}}(i)\). Since \(E_t^A\) is common across arms by definition, we define the UCB index as
\(
\widehat V_t(i)+B_t^{\mathrm{stat}}(i),
\)
and select
\(
a_t
\in
\arg\max_{i\in[K]}
\{
\widehat V_t(i)+B_t^{\mathrm{stat}}(i)
\}.
\)
All active agents share the same UCB index and therefore select the common arm
\(a_t\). Each agent then observes its local reward and updates the empirical
mean and pull count of the selected arm.

If the same arm remains globally optimal with a persistent gap, we say that
the system has a stable regime. In this regime, once the uncertainty
maintained by the policy becomes and remains sufficiently small w.r.t. the
gap, the stable optimal arm has the largest UCB index at every subsequent
round and is selected thereafter. The resulting behavior resembles
explore-then-commit. Although exploration is driven adaptively by the UCB
indices, not by a prescribed exploration phase, the policy eventually
enters an analogous commitment phase.

\subsection{Manifold-based Learning}
\label{sec:manifold-method}

A potential pain point of the preceding method is its reliance on a ``good
enough'' pre-training error, which may be out of reach in practice. For example,
the pre-training error may be unavailable, difficult to quantify a priori, or
just too large to provide useful information. It can then propagate to the population level, where the
decision becomes more delicate. However, ultimately, decisions depend on their aggregate rewards over the
population, not necessarily requiring control of agent-wise reward uncertainty. This motivates
us to \emph{exploit the structure of agent rewards and work directly
with uncertainty in the aggregate rewards}. Although arrivals and departures
change the active agents and their reward vectors, observations can transfer
across agents when they share a common geometric structure. Our handle is to
represent agents on a manifold, learn the reward function of each arm over this
manifold, and express its decision-relevant uncertainty through a functional
induced by the population.

Let \(\mathcal X\) be a known compact Riemannian manifold. Each agent
\(m\) has an observed representation \(x_m\in\mathcal X\), available upon
entry and fixed during its lifetime. For each arm \(i\in[K]\), let
\(f_i:\mathcal X\rightarrow[0,1]\) denote its unknown reward function, so that
\(
\mu_m^i=f_i(x_m).
\) Let \(k:\mathcal X\times\mathcal X\rightarrow\mathbb R\) be a known
positive-definite kernel with associated RKHS \(\mathcal H_k\), and we define
the canonical feature embedding as 
\(
\phi(x)=k(x,\cdot)\in\mathcal H_k.
\)
As is standard, we assume \(0 < k(x,x)\le 1\) for every \(x\) and
\(\|f_i\|_{\mathcal H_k}\le S\) for every arm \(i\). By the reproducing
property of the RKHS,
\(
f_i(x)
=
\langle f_i,\phi(x)\rangle_{\mathcal H_k}.
\)
In other words, the embedding connects the location of agent
\(m\) on the manifold to its reward.

We aggregate the feature embeddings of the active agents into
\(
\Psi_t
=
\sum_{m\in\mathcal M_t}\phi(x_m)
\in\mathcal H_k.
\)
We call \(\Psi_t\) the \emph{dynamic population embedding}. By linearity and
the reproducing property (also the Riesz representation theorem \citep{goodrich1970riesz}), we have that 
\[
\begin{aligned}
V_t(i)
&=
\sum_{m\in\mathcal M_t}f_i(x_m) =
\sum_{m\in\mathcal M_t}
\langle f_i,\phi(x_m)\rangle_{\mathcal H_k} =
\langle f_i,\Psi_t\rangle_{\mathcal H_k}.
\end{aligned}
\]
Hence, for evaluating arm \(i\), it is enough to learn how
\(f_i\) behaves on the current embedding \(\Psi_t\).

For each \(i\), we estimate \(f_i\) using kernel ridge regression:
\[
\widehat f_{i,t-1}
\in
\arg\min_{f\in\mathcal H_k}
\{
\sum_{\substack{s<t:\\a_s=i}}
\sum_{m\in\mathcal M_s}
(r_m^i(s)-f(x_m))^2
+
\|f\|_{\mathcal H_k}^2\}.
\]
Its normal equation is
\(\mathcal A_{i,t-1}\widehat f_{i,t-1}=b_{i,t-1}\), where
\(\mathcal A_{i,t-1}=
I+\sum_{s<t:a_s=i}\sum_{m\in\mathcal M_s}
\phi(x_m)\otimes\phi(x_m)\) and
\(b_{i,t-1}=
\sum_{s<t:a_s=i}\sum_{m\in\mathcal M_s}
r_m^i(s)\phi(x_m)\), with
\((u\otimes u)h=\langle h,u\rangle_{\mathcal H_k}u\).

The estimate of the current aggregate reward is then expressed as 
\(\widehat f_{i,t-1}
=
\mathcal A_{i,t-1}^{-1}b_{i,t-1}\) and \(
\widehat V_t(i)
=
\langle\widehat f_{i,t-1},\Psi_t\rangle_{\mathcal H_k}.
\)
Hence, the decision rule uses the estimate only through its
projection onto the dynamic population embedding. But more than that, we need to quantify how much uncertainty within the estimate. Precisely, for \(n\ge 0\), let
\(
\gamma_n
=
\sup_{x_1,\ldots,x_n\in\mathcal X}
\frac12\log\det(I_n+K_n),
\)
where \(K_n\) is the Gram matrix of
\(\{x_1,\ldots,x_n\}\). Intuitively, \(\gamma_n\) quantifies the maximum information about the reward function that one hopes to obtain from observations
at \(n\) agent representations; it naturally 
connects to mutual information as seen later. Let \(\bar\gamma_n\) be any known nondecreasing
sequence satisfying \(\gamma_n\le\bar\gamma_n\); the choice of this sequence depends on the kernel \(k(\cdot,\cdot)\).

For arm \(i\), let
\(
N_i(t-1)
=
\sum_{\substack{s<t:\\a_s=i}}M_s
\)
denote the number of observations collected by round \(t-1\).  For confidence level \(\delta\in(0,1)\), let us define the exploration parameter as 
\(
\beta_{i,t-1}
=
S+
R\sqrt{
2(
\bar\gamma_{N_i(t-1)}
+
1+
\log\frac{K}{\delta}
)
}
\) and the uncertainty in the aggregate reward $V_t(i)$ as 
\(
\sigma_{i,t}^{\mathrm{pop}}
=
\sqrt{
\langle
\Psi_t,
\mathcal A_{i,t-1}^{-1}\Psi_t
\rangle_{\mathcal H_k}
}.
\)
As shown in Section~\ref{sec:manifold-regret}, w.p. at least
\(1-\delta\),
\(
|
\widehat V_t(i)-V_t(i)
|
\le
\beta_{i,t-1}\sigma_{i,t}^{\mathrm{pop}}
\)
simultaneously over all arms and rounds. Each agent selects
\(
a_t
\in
\arg\max_{i\in[K]}
\{
\widehat V_t(i)
+
\beta_{i,t-1}\sigma_{i,t}^{\mathrm{pop}}
\}
\) following the UCB principle.
We then observe a batch of agent-level rewards and update
\(
\mathcal A_{a_t,t}
\)
and
\(
b_{a_t,t}.
\)

\emph{Finite-dimensional implementation.}
Although the method is formulated in the possibly infinite-dimensional RKHS
\(\mathcal H_k\), all decision-relevant quantities admit
finite-dimensional kernel representations.

Let us enumerate the \(N_i(t-1)\) observations of arm \(i\) as
\(\{(x_{i,\ell},y_{i,\ell})\}_{\ell=1}^{N_i(t-1)}\), and let
\(y_{i,t-1}\in\mathbb R^{N_i(t-1)}\) collect their rewards. Let us consider an operator which evaluates the values of any function in RKHS w.r.t. observation $x$ as 
\(
\Phi_{i,t-1}:\mathcal H_k\rightarrow\mathbb R^{N_i(t-1)},
\quad
[\Phi_{i,t-1}f]_\ell
=
f(x_{i,\ell}),
\)
with adjoint
\(
\Phi_{i,t-1}^{\ast}a
=
\sum_{\ell=1}^{N_i(t-1)}
a_\ell\phi(x_{i,\ell}),
\quad
a\in\mathbb R^{N_i(t-1)}.
\)
The Gram matrix is then
\(
K_{i,t-1}
=
\Phi_{i,t-1}\Phi_{i,t-1}^{\ast}. 
\) For any \(f\in\mathcal H_k\), by the reproducing property, we obtain 
\[
\Phi_{i,t-1}^{\ast}\Phi_{i,t-1}f
=
\sum_{\ell=1}^{N_i(t-1)}
\bigl(\phi(x_{i,\ell})\otimes\phi(x_{i,\ell})\bigr)f.
\]
Hence, we derive that 
\(
\mathcal A_{i,t-1}
=
I+\Phi_{i,t-1}^{\ast}\Phi_{i,t-1}.
\) Next we proceed to the representation of $b_{i,{t-1}}$. By the definition of the adjoint,
\(
\Phi_{i,t-1}^{\ast}y_{i,t-1}
=
\sum_{\ell=1}^{N_i(t-1)}
y_{i,\ell}\phi(x_{i,\ell})
=
b_{i,t-1}.
\) With these equivalences, we obtain 
\(
\widehat f_{i,t-1}
=
\left(
I+\Phi_{i,t-1}^{\ast}\Phi_{i,t-1}
\right)^{-1}
\Phi_{i,t-1}^{\ast}y_{i,t-1}.
\)
Using
\(
\left(
I+\Phi_{i,t-1}^{\ast}\Phi_{i,t-1}
\right)^{-1}
\Phi_{i,t-1}^{\ast}
=
\Phi_{i,t-1}^{\ast}
\left(
I+\Phi_{i,t-1}\Phi_{i,t-1}^{\ast}
\right)^{-1},
\) 
we obtain
\(
\widehat f_{i,t-1}
=
\Phi_{i,t-1}^{\ast}
(K_{i,t-1}+I)^{-1}
y_{i,t-1}.
\)
Let 
\(
q_{i,t}
=
\Phi_{i,t-1}\Psi_t
\) 
which collects the values of the population embedding $\Psi_t$ at the agent
representation histories $\{x_{i,l}\}$. Its \(\ell\)-th entry is
\(
[q_{i,t}]_\ell
=
\Psi_t(x_{i,\ell})
=
\sum_{m\in\mathcal M_t}k(x_{i,\ell},x_m).
\)
Moreover, we denote $c_t$ as the squared norm of the population embedding $\Psi_t$ as 
\(c_t
=
\|\Psi_t\|_{\mathcal H_k}^2
=
\sum_{m,n\in\mathcal M_t}k(x_m,x_n).
\)  Consequently, we obtain
\begin{align*}
& \widehat V_t(i)
=
\langle
\widehat f_{i,t-1},
\Psi_t
\rangle_{\mathcal H_k} =  \langle
\Phi_{i,t-1}^{\ast}
(K_{i,t-1}+I)^{-1}
y_{i,t-1},
\Psi_t\rangle_{\mathcal H_k} \\
& =
q_{i,t}^{\top}
(K_{i,t-1}+I)^{-1}
y_{i,t-1}.
\end{align*}

Using the Woodbury identity matrix \citep{henderson1981deriving} yields 
\(
\mathcal A_{i,t-1}^{-1}
=
I-
\Phi_{i,t-1}^{\ast}
(K_{i,t-1}+I)^{-1}
\Phi_{i,t-1}.
\) Hence, we express the radius as 
\begin{align*}
& \left(\sigma_{i,t}^{\mathrm{pop}}\right)^2
=
\left\langle
\Psi_t,
\mathcal A_{i,t-1}^{-1}\Psi_t
\right\rangle_{\mathcal H_k} \\ &=
\|\Psi_t\|_{\mathcal H_k}^2
-
(\Phi_{i,t-1}\Psi_t)^{\top}
(K_{i,t-1}+I)^{-1}
(\Phi_{i,t-1}\Psi_t)\\
&=
c_t
-
q_{i,t}^{\top}
(K_{i,t-1}+I)^{-1}
q_{i,t}.
\end{align*}

In other words, the kernel representation based on past observations is
algebraically equivalent to the infinite-dimensional counterpart above and
therefore yields the same UCB indices and arm choices, while providing a more
computationally appealing form. The full procedure is exhibited in
Algorithm~\ref{alg:manifold-functional-ucb}.

\subsubsection{Parallel Exploration}
\label{sec:manifold-heterogeneous-actions}

Thus far, all active agents select a common
arm at each round. This idea naturally aligns with the system-level objective
and is feasible under the complete communication graph. Under the manifold formulation, an observation from one agent can reduce
uncertainty in the aggregated reward, yet the extent of this
reduction varies across agents. This heterogeneity in information value creates an organic 
opportunity for parallel exploration. To this end, we next allow 
agents to diversify arms while aiming at the same target \(i_t^\star\in\arg\max_{i\in[K]}V_t(i)\). This creates extra freedom for information collection, in that agents can explore 
arms in parallel to jointly inform the 
decision, but of course as we shall foresee shortly, also an extra layer of complexity in exploration. 

For each arm \(i\), let
\(
\mathcal Q_{t,i}
=
\{m\in\mathcal M_t:a_m(t)=i\}
\)
denote the active agents assigned to arm \(i\), and let
\(M_{t,i}=|\mathcal Q_{t,i}|\). Then
\(\{\mathcal Q_{t,i}\}_{i\in[K]}\) partitions \(\mathcal M_t\), with
\(
\sum_{i=1}^K M_{t,i}=M_t.
\) Likewise, observations are accumulated from agents assigned to that arm,
so we can equivalently write 
\(
N_i(t-1)
=
\sum_{s<t}M_{s,i}
\), 
\(
\mathcal A_{i,t-1}
=
I+
\sum_{s<t}
\sum_{m\in\mathcal Q_{s,i}}
\phi(x_m)\otimes\phi(x_m)\) and \(
b_{i,t-1}
=
\sum_{s<t}
\sum_{m\in\mathcal Q_{s,i}}
r_m^i(s)\phi(x_m).
\)
The estimator
\(\widehat f_{i,t-1}=\mathcal A_{i,t-1}^{-1}b_{i,t-1}\)
induces the estimate
\(\widehat V_t(i)=
\langle\widehat f_{i,t-1},\Psi_t\rangle_{\mathcal H_k}\)
and uncertainty
\(\sigma_{i,t}^{\mathrm{pop}}\). Here the confidence bound continues to hold, i.e., w.h.p.
\(
\left|
\widehat V_t(i)-V_t(i)
\right|
\le
\beta_{i,t-1}\sigma_{i,t}^{\mathrm{pop}}. 
\)

We adapt top-two sampling to handle parallel exploration. Let us define the upper and lower confidence bounds  (UCB and LCB, respectively) of arms as 
\[
U_t(i)
=
\widehat V_t(i)
+
\beta_{i,t-1}\sigma_{i,t}^{\mathrm{pop}},
\quad
L_t(i)
=
\widehat V_t(i)
-
\beta_{i,t-1}\sigma_{i,t}^{\mathrm{pop}}.
\]
This leads us to the top-two arms as 
\[
h_t\in\arg\max_{i\in[K]}U_t(i),
\quad
\mathcal C_t
=
\{i\neq h_t:U_t(i)\ge L_t(h_t)\}.
\]
Intuitively, the top-one arm \(h_t\) maximizes the UCB index and is therefore a promising candidate for the best arm, while the \emph{arm set} \(\mathcal{C}_t\) contains arms whose UCB indices exceed the LCB of \(h_t\), and are thus also plausible candidates for the best arm. We retain $h_t$ (leader) and \(\mathcal{C}_t\) (challengers) to promote exploitation and exploration, respectively.

We next assign agents to either the leader or its plausible challengers. Let the exploration budget 
\(
e_t=\min\{M_t-1,|\mathcal C_t|\},
\)
so that at least one agent selects \(h_t\). When
\(\mathcal C_t=\emptyset\), \(e_t=0\) and all agents select \(h_t\). Suppose agent \(m\) were assigned to arm \(i\). Its observation would
add the rank-one term \(\phi(x_m)\otimes\phi(x_m)\) to the current design
operator $A_{i,t-1}$, yielding 
\(
\mathcal A_{i,t-1}^{(m)}
=
\mathcal A_{i,t-1}
+
\phi(x_m)\otimes\phi(x_m).
\) Correspondingly, we obtain
\((\sigma_{i,t}^{\mathrm{pop},(m)})^2
=
\langle
\Psi_t,
(\mathcal A_{i,t-1}^{(m)})^{-1}
\Psi_t\rangle_{\mathcal H_k}.
\)
Thus, the uncertainty reduction in \(V_{i,t}\) from this assignment is 
\(
G_{i,t}(m)
=
(\sigma_{i,t}^{\mathrm{pop}})^2
-
(\sigma_{i,t}^{\mathrm{pop},(m)})^2.
\)
By the Sherman-Morrison identity \citep{sherman1950adjustment},
\[
G_{i,t}(m)
=
\nicefrac{
\langle
\Psi_t,
\mathcal A_{i,t-1}^{-1}\phi(x_m)
\rangle_{\mathcal H_k}^{2}
}{
1+
\langle
\phi(x_m),
\mathcal A_{i,t-1}^{-1}\phi(x_m)
\rangle_{\mathcal H_k}
}.
\]

We select \(e_t\) agent-arm pairs
\(\mathcal E_t\subseteq\mathcal M_t\times\mathcal C_t\)
by maximizing the total uncertainty reduction,
\(
W_t(\mathcal E_t)
=
\sum_{(m,i)\in\mathcal E_t} G_{i,t}(m).
\) This is a bipartite matching problem, which can be solved independently as a
network flow problem \citep{goldberg1989network}. Every agent not included in
the matching \(\mathcal E_t\) selects \(h_t\).

\section{Theoretical Analyses}

\subsection{Agent-wise Global UCB}

We first characterize the regret bound under
Algorithm~\ref{alg:general-open-mamab-global}, which seeks to estimate the agent-wise mean rewards. For each agent \(m\), let
\(
N_{m,i}^{R}(T)
=
\sum_{t=1}^T
\mathbf{1}\{m\in\mathcal R_t,\ a_t=i\},\) and \(L_m=\sum_{i=1}^K N_{m,i}^{R}(T),
\)
which denote the total number of post-entry observations collected from agent \(m\) for arm \(i\) and
across all arms, respectively. In other words, $L_m$ is the lifetime of agent $m$. 
For any confidence level \(0 < \delta \leq 1\) and integer $n$, let us define
\(
\Lambda_m(n)
=
\log(
\nicefrac{2Kj(m)(j(m)+1)n(n+1)}{\delta}).
\)

\begin{theorem}
\label{thm:pretrained_general}
With Algorithm~\ref{alg:general-open-mamab-global}, with probability at
least \(1-\delta\), \refstepcounter{equation}\label{eq:theorem1}
\[R_T
\le
2\sum_{t=1}^T\sum_{m\in\mathcal A_t}\bar P_m
+
2\sum_{m:L_m>0}
\min\!\left\{
L_m\bar P_m,\;
K\bar P_m+\sqrt{2K\Lambda_m(L_m)L_m}
\right\}
\tag{\theequation}\]
\end{theorem}

The bound reveals two complementary sources of complexity. New agents always bring in error $P_m$, amplified by $\mathcal{A}_t$ over time upon entry due to reward heterogeneity. 
During the learning process, when the initial reward estimate is sufficiently accurate, its contribution to regret is further 
controlled by \(L_m\bar P_m\). Otherwise, observations collected after entry
progressively refine the estimate, leading to the statistical term
\(\sqrt{K\Lambda_m(L_m)L_m}\). Population dynamics affect the bound through
both the agents that enter the system $\mathcal{A}_t$ and each agent's lifetime 
\(L_m\). 
The exact order depends on \(\hat\mu_{m,i}^{\mathrm{entry}}\) and
\(\bar P_m\), which arise in common models such as linear and cluster models;
we refer to Table~\ref{tab:entry-information} in Appendix~\ref{app:special-cases}.

Indeed, the linear dependence on the pre-training error, the first term in \eqref{eq:theorem1}, is unavoidable in a heterogeneous open system. In a closed system, $|\mathcal{A}_t|=0$, and in the worst case, $\bar{P}_m=1$, so the first term does not enter the rate, and the dominant term in the regret bound is $O(\sqrt{T})$, consistent with the current literature. In a homogeneous open system, $\bar{P}_m$ is the current confidence radius shared by all agents, which shrinks over time. With proper sampling that pools observations across agents in the algorithmic step to optimize sample complexity, the algorithm may recover the rate in \citep{xu2026open}. The second term captures the post-entry learning cost for continuing agents. It cannot be removed, as it matches the \(O(\sqrt{KT})\) order in closed systems and also holds in general up to a logarithmic factor in \(T\). To see why, we consider two
distinct hard instances. First, newly arriving agents may remain for only
one round, so their rewards affect the current decision before any online
observations are available. Second, agents may appear in disjoint
episodes with unrelated reward vectors, forcing the learner to solve a fresh
bandit problem after each arrival. The theorem formalizes the idea.

\begin{theorem}
\label{thm:pivotal-arrival-lower-bound}
Let us assume \(K=2\) and a sequence of arrivals \(\{\mathcal A_t\}_{t=1}^T\) and numbers
\(\{p_m\in[0,1]:m\in\mathcal A_t,\ t\in[T]\}\). 
There exists instances
\(\mathcal I(\{p_m\})\) with 
\(
P_m
=
p_m
\) and instances
\(\mathcal Q(\{L_m,p_m\})\) with $\sum_{m}(L_m+1) \leq T$. 
For any random policy \(\pi\) with coordinated actions,
\(\sup_{\nu\in\mathcal I(\{p_m\})}
\mathbb E_\nu^\pi[R_T] \geq 
\nicefrac{1}{2}
\sum_{t=1}^T
\sum_{m\in\mathcal A_t}
\bar P_m\) and \( 
\sup_{\nu\in\mathcal Q(\{L_m,p_m\})}
\mathbb E_\nu^\pi[R_T]
\geq 
\nicefrac{1}{128}
\sum_{m: L_m > 0}
\min\{
L_m p_m,\sqrt{K L_m}\}.\) 
\end{theorem}

\subsubsection{Sharper Bounds under a Stable Arm}
\label{sec:stable-global-arm}
Although agents may be heterogeneous, the aggregate population reward admits
a persistent optimal arm.  
This can arise, for example, when the population is sufficiently large relative
to new arrivals after certain time steps and the separation between arms is sufficiently large that the
additional reward heterogeneity introduced by new arrivals does not overturn
the optimal arm.  We formalize this favorable regime as the following assumption on the problem instance. 

\begin{assumption}
\label{ass:persistent-global-optimum}
There exist \(i^\dagger\in[K]\), \(\Delta>0\), and
\(\tau_{\mathrm{stab}}\le T\) s.t. 
\(
\bar V_t(i^\dagger)
-
\max_{i\neq i^\dagger}\bar V_t(i)
\ge \Delta,
\quad
\forall t\ge\tau_{\mathrm{stab}}.
\)
\end{assumption}

However, the existence of such an arm is only part of the story, as the performance also depends on whether the method is able to identify this arm, which is characterized by another burn-in period. Let
\(
W_t(i)
=
B_t^{\mathrm{stat}}(i)+E_t^A\)
and \(\omega_t(i)
=
\nicefrac{W_t(i)}{M_t},
\)
and let us define
\(
\tau_{\mathrm{id}}
=
\inf\{
t\in[T]:
\max_{i\in[K]}\omega_s(i)
\le
\nicefrac{\Delta}{4},
\quad
\forall s\in\{t,\ldots,T\}\},
\) 
with \(\inf\emptyset=T+1\). Intuitively, it is the burn-in time for the uncertainty radius shrinks below the uniform gap $\Delta$ so that Algorithm \ref{alg:general-open-mamab-global} can identify this optimal arm. Finally, let us set $\tau$ as
\(
\tau
=
\max\{\tau_{\mathrm{stab}},\tau_{\mathrm{id}}\},
\)
after which we find the persistently optimal arm. The regret bound reads as follows. 

\begin{theorem}
\label{thm:stable-global-arm}
Suppose Assumption~\ref{ass:persistent-global-optimum} holds. Run
Algorithm~\ref{alg:general-open-mamab-global} with confidence parameter
\(\delta/2\). Then, with probability at least \(1-\delta\),
\(
a_t=i^\dagger,\) for any \(t \geq \tau\).
On the same event,
\(
\bar R_T
\le
\tau-1\) and \(R_T
\le
\sum_{t=1}^{\tau-1}M_t.
\) We denote \(\tau=T+1\) when \(\tau > T\). 
\end{theorem}

Once $t \geq \tau$, regret stops to grow. The rate of course depends on agent patterns 
\(M_t\), pre-training errors ($E_t^{A}$), and the burn-in time \(\tau\), and is thus instance-dependent. As an example, if
\(M_t=\Theta(t)\), \(|\mathcal A_t|\le\bar A\), and \(\bar P_m\le C_Pt^{-\alpha}\) with $\alpha>0$, then
\(
\tau \le \tau_{\mathrm{cf}}=O(\log T)
\) under Assumption \ref{ass:stationary-arrival-composition}. If the remaining problem parameters are independent of \(T\), 
\(
\bar R_T=O(\log T)\) and \(
R_T=O\left(M_0\log T+\bar A\log^2 T\right).
\)
Thus, they recover the rate in closed systems \citep{landgren2016distributed} and in homogeneous open systems with Poisson processes \citep{xu2026open}, implying consistency while being applicable to more general regimes. Details on $\tau$ are in Appendix~\ref{app:tau-characterization}. 

\subsection{Manifold-Based Learning}
\label{sec:manifold-regret}

We first characterize the regret of
Algorithm~\ref{alg:manifold-functional-ucb} for a general kernel, and then
specialize the result to an intrinsic Mat\'ern kernel to expose the dependence
on the intrinsic dimension. Foremost, our construction of the estimate and confidence radius is valid w.r.t. $V_t(i)$. 
\begin{lemma}
\label{lem:population-functional-confidence}
Under the assumptions of
Section~\ref{sec:manifold-method}, w.p.  at least \(1-\delta\), for every \(t\in[T]\) and \(i\in[K]\),
\(
|
\widehat V_t(i)-V_t(i)
|
\le
\beta_{i,t-1}\sigma_{i,t}^{\mathrm{pop}}.
\)
\end{lemma}

Then, controlling regret means controlling the radius. For arm \(i\), let us define its information gain by \(\Gamma_{i,t}=\frac12\log\det(I+K_{i,t})\) where \(K_{i,t}\) is the aforementioned Gram matrix. Let \(\Delta\Gamma_t=\Gamma_{a_t,t}-\Gamma_{a_t,t-1}\) denote the information gained at round \(t\). Importantly, we have \(\left(\sigma_{a_t,t}^{\mathrm{pop}}\right)^2 \le 2\nicefrac{M_t^2}{\log(1+M_t)}\,\Delta\Gamma_t\). Let us define \(\Gamma_T^\Sigma =\sum_{t=1}^T\Delta\Gamma_t=\sum_{i=1}^K\Gamma_{i,T}\), \(N_T=\sum_{t=1}^TM_t\), and  \(H_T=\sum_{t=1}^T\nicefrac{M_t^2}{\log(1+M_t)}\). Hence, \(H_T\) records the total amplification of uncertainty along the population path, \(\Gamma_T^\Sigma\) records the cumulative information gained to resolve uncertainty, and $N_T$ is the total sample count. 

\begin{theorem}
\label{thm:manifold-regret}
Under the assumptions of
Section~\ref{sec:manifold-method}, w.p. at least \(1-\delta\),
\(
R_T
\le
2\beta_T
\sqrt{
2\Gamma_T^\Sigma H_T
}
\) with \(
\beta_{T}
=
S+
R\sqrt{
2\left(
\bar\gamma_{N_T}
+
1+
\log\nicefrac{K}{\delta}
\right)
}.
\)
By 
\(
\Gamma_T^\Sigma
\le
\sum_{i=1}^K\bar\gamma_{N_i(T)}
\le
K\bar\gamma_{N_T},
\)
we have \(
R_T
\le
2\beta_T
\sqrt{
2K\bar\gamma_{N_T}H_T
}.
\)
\end{theorem}

Theorem~\ref{thm:manifold-regret} exhibits a clean interaction between reward heterogeneity and population dynamics. Particularly, the term
\(\Gamma_T^\Sigma\) captures the effective
complexity of the reward landscape beyond the number of agents or
the ambient dimension. The square root rate
\(\sqrt{\Gamma_T^\Sigma H_T}\) follows from summing the confidence widths over
time and applying Cauchy--Schwarz. Notably, there is another $\bar{\gamma}_{N_T}$ dependency through $\beta_T$, and thus the total rate is of \(\sqrt{H_T}\bar{\gamma}_{N_T}\) consistent with the existing analysis machinery in RKHS.  Regret is small when either the 
reward landscape is easy to learn, or
the population trajectory limits the exposure of the uncertainty.

With an intrinsic Mat\'ern kernel that exhibits a well understood spectrum, the regret dependence on the intrinsic dimension and smoothness becomes explicit. Assume that the agent representation space \(\mathcal X\) is a smooth,
compact, connected \(d\)-dimensional manifold without boundary, equipped with
a Riemannian metric \(g\). Let
\(\Delta_g\) denote the Laplace-Beltrami operator \citep{rosenberg1997laplacian}, the intrinsic analogue of
the Euclidean Laplacian on the manifold, and let
\(\{(\lambda_j,\varphi_j)\}_{j=0}^{\infty}\) denote its eigenpairs:
\(
-\Delta_g\varphi_j=\lambda_j\varphi_j
\) with $0=\lambda_0\le\lambda_1\le\cdots.$
Here, the eigenfunctions \(\varphi_j\) describe spatial modes over the
manifold, while the eigenvalues \(\lambda_j\) characterize their spatial
frequencies. For \(\nu>0\) and \(\ell>0\), we define \(a_\ell=2\nu/\ell^2\) and consider the
normalized intrinsic Mat\'ern kernel
\(
k_{\nu,\ell}(x,x')
=
\nicefrac{1}{\kappa_{\nu,\ell}}
\sum_{j=0}^{\infty}
(a_\ell+\lambda_j)^{-\nu-d/2}
\varphi_j(x)\varphi_j(x'),
\)
where constant $\kappa_{\nu,\ell} = \sup_{x}\sum_{j=0}^{\infty}
(a_\ell+\lambda_j)^{-\nu-d/2}
\varphi_j(x)\varphi_j(x)$  s.t. 
\(k_{\nu,\ell}(x,x)\le1\) for every \(x\in\mathcal X\). We then have: 

\begin{theorem}
\label{thm:manifold-matern-regret}
For the Mat\'ern kernel, w.p. at least
\(1-\delta\),
\[
R_T
\le
2\bar\beta_T
\sqrt{
2C_\gamma
K^{\frac{2\nu}{2\nu+d}}
N_T^{\frac{d}{2\nu+d}}
L_T^{\frac{2\nu}{2\nu+d}}
H_T
},
\]
where
\(
\bar\beta_T
=
S+
R\sqrt{
2[
C_\gamma
N_T^{\frac{d}{2\nu+d}}
L_T^{\frac{2\nu}{2\nu+d}}
+
1+
\log\frac{K}{\delta}]
} 
\) and \(
L_T=\log(e+N_T).
\) 
\end{theorem}

Moreover, under bounded population size \(M_t\le\overline M\) for all \(t\), we have 
\(N_T\le\overline MT\) and
\(
H_T
\le
\nicefrac{
T\overline M^2
}{
\log(1+\overline M)
}.
\)
Hence, up to logarithmic factors and constants w.r.t.
\((\mathcal X,g)\), \(\nu\), \(\ell\), \(S\), \(R\), and
\(\overline M\),
\(
R_T
=
\widetilde O(
K^{\frac{\nu}{2\nu+d}}
T^{\frac12+\frac{d}{2\nu+d}}
)
\). In particular, it is sublinear in \(T\)
if \(\nu>d/2\).

\subsubsection{Parallel Exploration}
\label{sec:manifold-heterogeneous-regret}

Let us now echo back to the method where we allow agents to diversify arm selection. Recall that  \(e_t\) is the number of agents who explore $\mathcal{C}_t$; $e_t = 0$ in the homogeneous exploration case. This distinction yields: 

\begin{theorem}
\label{thm:manifold-heterogeneous-regret}
Under the assumptions of
Section~\ref{sec:manifold-heterogeneous-actions}, with probability at least
\(1-\delta\),
\(
R_T
\le
2\beta_T
\sqrt{
2\Gamma_T^\Sigma H_T^{\mathrm{het}}
}
+
(2\beta_T+1)N_T^{\mathrm{exp}},
\)
where
\(
N_T^{\mathrm{exp}}
=
\sum_{t=1}^T e_t\) and \(
H_T^{\mathrm{het}}
=
\sum_{t=1}^T
\frac{(M_t-e_t)^4}
{M_t^2\log(1+M_t-e_t)}. 
\) If, in addition, parallel exploration is restricted to assignments
\((m,i)\) whose one-step information gain $I_{i,t}(m) = \nicefrac12
\log(
1+\langle
\phi(x_m),
\mathcal A_{i,t-1}^{-1}\phi(x_m)
\rangle_{\mathcal H_k})$ is at least \(\eta>0\), then
\(
N_T^{\mathrm{exp}}
\le
\nicefrac{\Gamma_T^\Sigma}{\eta},
\)
and 
\(
R_T
\le
2\beta_T
\sqrt{
2\Gamma_T^\Sigma H_T^{\mathrm{het}}
}
+
\frac{2\beta_T+1}{\eta}\Gamma_T^\Sigma.
\)
Lastly,
\(
\Gamma_T^\Sigma
\le
K\bar\gamma_{N_T}
\) which allows us to plug in $\bar\gamma_{N_T}$. 
\end{theorem}

By the second argument in Theorem~\ref{thm:manifold-heterogeneous-regret}, the rate of parallel exploration depends on the matching mechanism for the budget \(e_t\), which is our next focus. For \(e_t>0\), let \(\eta_t^\star = \max_{\mathcal P\in\mathcal{MT}_t(e_t)} \min_{(m,i)\in\mathcal P} I_{i,t}(m)\) denote the largest one-step information threshold for which there exists a feasible matching whose every assigned pair attains at least that threshold. Indeed, \(\mathcal A_{i,t-1}\succ0\) and \(\|\phi(x_m)\|_{\mathcal H_k}^2=k(x_m,x_m)>0\) imply \(\left\langle \phi(x_m), \mathcal A_{i,t-1}^{-1}\phi(x_m) \right\rangle_{\mathcal H_k}>0\), and hence \(I_{i,t}(m)>0\) for every feasible agent--arm pair. Since the set of feasible matchings \(\mathcal{MT}_t(e_t)\) is finite and nonempty, it follows that \(\eta_t^\star>0\). For \(0<\eta\le\eta_t^\star\), define the level set \[\mathcal{MT}_t^\eta(e_t) = \left\{ \mathcal P\in\mathcal{MT}_t(e_t): I_{i,t}(m)\ge\eta \text{ for every }(m,i)\in\mathcal P \right\}\]. We then choose \(\mathcal P_t^{\star,\eta} \in \arg\max_{\mathcal P\in\mathcal{MT}_t^\eta(e_t)} W_t(\mathcal P)\). Again since \(\mathcal{MT}_t(e_t)\) is finite, there exists \(\widetilde{\mathcal P}_t\in\mathcal{MT}_t(e_t)\) s.t. \(\min_{(m,i)\in\widetilde{\mathcal P}_t} I_{i,t}(m) = \eta_t^\star\). Therefore, for every \(0<\eta\le\eta_t^\star\), \(\widetilde{\mathcal P}_t\in\mathcal{MT}_t^\eta(e_t)\), and hence \(\mathcal{MT}_t^\eta(e_t)\neq\emptyset\). That said, \(\mathcal P_t^{\star,\eta}\) is well defined and meets the one-step information threshold.

We note that this matching itself may be suboptimal compared with the original optimal matching \(\arg\max_{\mathcal P\in\mathcal{MT}_t(e_t)} W_t(\mathcal P)\) over all possible matchings, due to the information-gain constraint. This further shows the intrinsic  entanglement between uncertainty reduction (w.r.t. decision-centric quantity) and information gain (w.r.t. agent-level exploration). However, it is indeed possible for the two matchings to coincide, i.e., to achieve the best of both worlds. To see this, by Cauchy-Schwarz, \refstepcounter{equation}\label{eq:information-functional-relation}
\(G_{i,t}(m) \le (\sigma_{i,t}^{\mathrm{pop}})^2 (1-e^{-2I_{i,t}(m)})\)
~(\theequation), i.e., \(\nicefrac{G_{i,t}(m)}{(\sigma_{i,t}^{\mathrm{pop}})^2} \ge 1-e^{-2\eta} \Longrightarrow I_{i,t}(m)\ge\eta\). Thus, if the original optimal matching meets the condition on \(G_{i,t}(m)\), then it belongs to \(\mathcal{MT}_t^\eta(e_t)\) and is also the maximizer of $W_t(\mathcal{P})$ therein. 

\paragraph{Homogeneous versus Parallel}
If no exploratory assignments are made, so that \(e_t=0\) for every \(t\), then \(M_t-e_t=M_t\), \(D_T=N_T^{\mathrm{exp}}=0\), and \(H_T^{\mathrm{het}}=\sum_{t=1}^T \frac{M_t^2}{\log(1+M_t)}=H_T\). Theorem~\ref{thm:manifold-heterogeneous-regret} therefore reduces to the regret bound of Theorem~\ref{thm:manifold-regret}, implying consistency. Generally, if \(e_t>0\), the regret bound achieves a reduction in the first term by replacing \(H_T\) with \(H_T^{\mathrm{het}}\), but at the price of the extra exploration reflected in the second term.  The bound for parallel exploration is smaller whenever
\(
2\beta_T\sqrt{2(K\bar{\gamma}_{N_T})}
(
\sqrt{H_T}-\sqrt{H_T^{\mathrm{het}}}
)
>
\nicefrac{(2\beta_T+1)}{\eta}(K\bar{\gamma}_{N_T}). 
\)
Since \(\eta_t^\star\) is the largest threshold over all feasible matchings of size \(e_t\), it is equivalent to 
\[
\min_{t:e_t>0}\eta_t^\star
>
\nicefrac{(2\beta_T+1)}{2\sqrt{2}\beta_T}
\nicefrac{\sqrt{K\bar{\gamma}_{N_T}}}
{(\sqrt{H_T}-\sqrt{H_T^{\mathrm{het}}})},
\]
provided \(H_T^{\mathrm{het}}<H_T\). Increasing \(e_t\) reduces \(H_T^{\mathrm{het}}\), reflecting
more parallel exploration and favorably lowering the threshold on $\eta_t^\star$. However, it may also decrease
\(\eta_t^\star\). This presents a nontrivial trade-off, particularly without the hindsight information about $\min_{t:e_t>0}\eta_t^\star$.

\subsection{Computational Complexity}
While we see a progressive improvement in regret,
except for the trade-offs between homogeneous and heterogeneous
exploration, their computational complexity might also be a valid
factor to take into account at deployment time. We observe a
progressive increase in this complexity from agent-wise learning to
manifold-based learning, while the two manifold-based methods are more or less
equivalent. While less desirable, this is perhaps the price we pay for the fact that
\emph{``all past observations remain useful even after agents have
departed given the manifold.''} We hope to convey this. 

Specifically, Agent-wise Global-UCB has lower computational overhead. At round
\(t\), it takes \(O(KM_t)\) to compute the estimates and confidence radii of
all \(K\) arms over the \(M_t\) agents. The manifold methods inherit the computational cost of kernel
regression. Recall that \(N_i(t-1)\) and \(N_{t-1}\) are the numbers of
observations of arm \(i\) and all arms, respectively, by \(t-1\); thus $N_{t-1} = \sum_{s=1}^{t-1}M_s$. Evaluating the UCBs of
all arms takes
\(
O(
N_{t-1}M_t+M_t^2+
\sum_{i=1}^K N_i(t-1)^2
)
\)
time. Here, the first two terms arise from constructing the kernel quantities
\(q_{i,t}\) and \(c_t\), while the last term is the cost of computing
\((K_{i,t-1}+I)^{-1}q_{i,t}\) when evaluating
\(\sigma_{i,t}^{\mathrm{pop}}\). After observing the batch of \(M_t\) agents, the kernel matrix of the selected
arm grows from dimension \(N_{a_t}(t-1)\) to
\(N_{a_t}(t-1)+M_t\). Reusing the inverse of the kernel matrix from the previous round, incorporating
one new observation requires updating the inverse after adding one new row and
column. This update involves \(O(n^2)\) entries when the kernel matrix
has size \(n\times n\). Summing this cost
over the \(M_t\) new observations gives
\(O(M_t(N_{a_t}(t-1)+M_t)^2)\). Thus, the overall per-round complexity is
\(O(M_t(N_{t-1}+M_t)^2) = O(M_t(\sum_{s=1}^{t}M_s)^2)\) in the worst case. Parallel exploration incurs extra cost of evaluating the assignments and solving the matching problem; however, none 
dominates \(O(M_t(\sum_{s=1}^{t}M_s)^2)\) in the worst case. More details are in Appendix \ref{app:complexity}.

\section{Conclusions}
Learning in open systems is intricate under reward heterogeneity and population
dynamics, and the findings herein show different difficulties of openness. In the general setting, new agents contribute uncertainty through their
pre-training errors, so regret depends on both these errors and their lifetimes
in the system, with these dependencies being individually minimax tight up to logarithmic factors. When a stable optimal arm exists,  regret is
instead governed by the time needed to identify the persistent optimum. Both depend on population dynamics. Through the manifold, observations transfer across agents via their
common geometry, replacing separate control of agent-wise uncertainty with
direct control of uncertainty in the aggregate reward. The regret then depends
on the cumulative information gain and the population path, coupling the
complexity of learning with the population dynamics that scale
uncertainty. 

\clearpage 
\newpage 

\section*{AI Use Statement}

Generative AI tools, particularly GPT-5.6 Sol, have assisted the authors with manuscript polishing, drafting the proofs of the mathematical claims, and reviewing the content. The tool does not participate in the ideation and conceptualization process of the paper. The authors have independently revised the content and take responsibility for the final content of this work, including text, claims, or artifacts produced with the aid of generative AI.

\bibliographystyle{plainnat}
\bibliography{main}

\clearpage 
\newpage 


\appendix
\section{Appendix}

\subsection{Pseudo code}

\begin{algorithm2e}[h]
\caption{Agent-Wise Global-UCB}
\label{alg:general-open-mamab-global}
\DontPrintSemicolon
\KwIn{Arms \(K\), horizon \(T\), confidence level \(\delta\)}
\KwOut{Common arm choices \(\{a_t\}_{t\in[T]}\)}
\For{\(t=1,2,\ldots,T\)}{
Observe \(\mathcal M_t\), and set
\(\mathcal A_t=\mathcal M_t\setminus\mathcal M_{t-1}\) and
\(\mathcal R_t=\mathcal M_t\cap\mathcal M_{t-1}\)\;

\ForEach{\(m\in\mathcal A_t\)}{
Assign its arrival index \(j(m)\), obtain
\(\{\hat\mu_{m,i}^{\mathrm{entry}}\}_{i\in[K]}\) and \(\bar P_m\), and set
\(n_{m,i}(t-1)=0\)\;
}

\ForEach{\(m\in\mathcal R_t,\ i\in[K]\)}{
Set
\(\rho_{m,i}(t-1)\leftarrow
\min\{\bar P_m,c_m(n_{m,i}(t-1))\}\), and set
\(\hat\mu_{m,i}(t-1)\leftarrow\hat\mu_{m,i}^{\mathrm{entry}}\) if
\(\bar P_m\le c_m(n_{m,i}(t-1))\), and
\(\hat\mu_{m,i}(t-1)\leftarrow\bar r_{m,i}(t-1)\) otherwise\;
}

\ForEach{\(i\in[K]\)}{
Set
\(\widehat V_t(i)\leftarrow
\sum_{m\in\mathcal R_t}\hat\mu_{m,i}(t-1)
+\sum_{m\in\mathcal A_t}\hat\mu_{m,i}^{\mathrm{entry}}\)
and
\(B_t^{\mathrm{stat}}(i)\leftarrow
\sum_{m\in\mathcal R_t}\rho_{m,i}(t-1)\)\;
}

Select
\(a_t\in\arg\max_{i\in[K]}
\{\widehat V_t(i)+B_t^{\mathrm{stat}}(i)\}\)\;

\ForEach{\(m\in\mathcal M_t\)}{
Pull \(a_t\), observe \(r_m^{a_t}(t)\), and update
\(n_{m,a_t}(t)\) and \(\bar r_{m,a_t}(t)\)\;
}
}
\end{algorithm2e}

\begin{algorithm2e}[h]
\caption{Manifold-based Learning}
\label{alg:manifold-functional-ucb}
\DontPrintSemicolon
\KwIn{Kernel \(k\), RKHS bound \(S\), noise parameter \(R\),
confidence level \(\delta\), information-gain bound
\(\{\bar\gamma_n\}_{n\ge0}\)}
\KwOut{Common arm choices \(\{a_t\}_{t\in[T]}\)}

Initialize \(\mathcal D_i\leftarrow\emptyset\) and \(N_i(0)\leftarrow0\)
for every \(i\in[K]\)\;

\For{\(t=1,2,\ldots,T\)}{
Observe \(\mathcal M_t\) and \(\{x_m:m\in\mathcal M_t\}\), and set
\(c_t\leftarrow\sum_{m,n\in\mathcal M_t}k(x_m,x_n)\)\;

\ForEach{\(i\in[K]\)}{
Enumerate
\(\mathcal D_i=\{(x_{i,\ell},y_{i,\ell})\}_{\ell=1}^{N_i(t-1)}\),
and form the Gram matrix
\([K_{i,t-1}]_{\ell\ell'}\leftarrow
k(x_{i,\ell},x_{i,\ell'})\)\;

Set
\([q_{i,t}]_\ell\leftarrow
\sum_{m\in\mathcal M_t}k(x_{i,\ell},x_m)\)\;

Set
\(\widehat V_t(i)\leftarrow
q_{i,t}^{\top}(K_{i,t-1}+I)^{-1}y_{i,t-1}\)
and
\(\sigma_{i,t}^{\mathrm{pop}}\leftarrow
\sqrt{c_t-q_{i,t}^{\top}(K_{i,t-1}+I)^{-1}q_{i,t}}\)\;

Set
\(\beta_{i,t-1}\leftarrow
S+R\sqrt{2\left(
\bar\gamma_{N_i(t-1)}+1+\log(K/\delta)
\right)}\)\;
}

Select
\(a_t\in\arg\max_{i\in[K]}
\{\widehat V_t(i)+
\beta_{i,t-1}\sigma_{i,t}^{\mathrm{pop}}\}\)\;

Every \(m\in\mathcal M_t\) pulls \(a_t\) and observes
\(r_m^{a_t}(t)\)\;

Update
\(\mathcal D_{a_t}\leftarrow
\mathcal D_{a_t}\cup
\{(x_m,r_m^{a_t}(t)):m\in\mathcal M_t\}\)
and
\(N_{a_t}(t)\leftarrow N_{a_t}(t-1)+M_t\)\;
}
\end{algorithm2e}

\subsection{Special cases for Theorem 1}\label{app:special-cases}

\begin{table}[H]
\centering
\small
\begin{tabular}{lll}
\toprule
Model & Entry estimate & Valid entry radius \(\bar P_m\) \\
\midrule
Linear
&
\(x_m^\top\hat\theta_i\)
&
\(\displaystyle
\max_{i\in[K]}
|x_m^\top(\hat\theta_i-\theta_i)|
\)
\\[1.2ex]

Nonlinear parametric
&
\(f_i(x_m,\hat\theta_i)\)
&
\(\displaystyle
\max_{i\in[K]}
L_i\|\hat\theta_i-\theta_i\|_2
\)
\\[1.2ex]

Clustered
&
\(\hat\theta_{c(m)}^i\)
&
\(\displaystyle
\max_{i\in[K]}
|\hat\theta_{c(m)}^i-\theta_{c(m)}^i|
\)
\\[1.2ex]

Zero knowledge
&
\(0\)
&
\(1\)
\\
\bottomrule
\end{tabular}
\caption{Examples of entry information for newly arriving agents.}
\label{tab:entry-information}
\end{table}

In the linear and nonlinear cases, estimates of shared parameters can be
transferred to a newly arriving agent through its observed features; in the
clustered case, the agent inherits information associated with its cluster;
and zero-knowledge initialization provides the worst-case valid radius under
rewards in \([0,1]\). Once \(\bar P_m\) is specified, all cases are covered by
the same regret bound in Theorem~\ref{thm:pretrained_general}.

\subsection{Characterization of \(\tau\)}
\label{app:tau-characterization}

Recall that
\[
\tau
=
\max\{\tau_{\mathrm{stab}},\tau_{\mathrm{id}}\},
\]
where \(\tau_{\mathrm{stab}}\) is the time after which the persistent global
optimum in Assumption~\ref{ass:persistent-global-optimum} is established, and
\(\tau_{\mathrm{id}}\) is the time after which the uncertainty in the
Global-UCB indices is too small to overturn this ordering. In general,
\(\tau_{\mathrm{stab}}\) is instance-dependent. We first give one sufficient
condition under which it admits an explicit bound and then characterize
\(\tau_{\mathrm{id}}\).

\paragraph{One sufficient condition for population stability.}

The persistent-optimum condition does not require any particular distribution
of the agent reward vectors. The following stationary population model provides
one sufficient condition under which it holds with high probability.

\begin{assumption}[Stationary population composition]
\label{ass:stationary-arrival-composition}
There are no departures, and the reward-mean vectors
\(
\mu_m=(\mu_m^1,\ldots,\mu_m^K)\in[0,1]^K
\)
of all initial and arriving agents are drawn independently from a common
distribution, independently of the arrival process. Let
\(
v(i)=\mathbb E[\mu_m^i],
\)
and suppose that \(v\) has a unique maximizer \(i^\dagger\in[K]\), with gap
\(
\Delta_0
=
v(i^\dagger)
-
\max_{i\neq i^\dagger}v(i)
>0.
\)
\end{assumption}

Fix
\(
0<\Delta<\Delta_0,
\)
and define
\(
N_{\mathrm{stab}}
=
\frac{2}{(\Delta_0-\Delta)^2}
\log\frac{4KT}{\delta}.
\)
Let
\[
\tau_{\mathrm{pop}}
=
\inf\left\{
t\in[T]:
M_s\ge N_{\mathrm{stab}},
\quad
\forall s\in\{t,\ldots,T\}
\right\},
\]
with \(\inf\emptyset=T+1\). Under
Assumption~\ref{ass:stationary-arrival-composition}, there are no departures,
so \(M_t\) is nondecreasing and
\[
\tau_{\mathrm{pop}}
=
\inf\left\{
t\in[T]:
M_t\ge N_{\mathrm{stab}}
\right\}.
\]

\begin{lemma}[Population stability]
\label{lem:population-stability}
Under Assumption~\ref{ass:stationary-arrival-composition}, with probability at
least \(1-\delta/2\),
\[
\bar V_t(i^\dagger)
-
\max_{i\neq i^\dagger}\bar V_t(i)
\ge
\Delta,
\qquad
\forall t\ge\tau_{\mathrm{pop}},\ t\le T.
\]
Consequently, on this event,  if $\tau_{\mathrm{pop}}\leq T$, 
Assumption~\ref{ass:persistent-global-optimum} holds with
\(
\tau_{\mathrm{stab}}
\le
\tau_{\mathrm{pop}}.
\)
\end{lemma}

\begin{proof}
For a fixed $t$ and $i$, conditional on $M_t$,
$\bar V_t(i)$ is the average of $M_t$ independent random variables in
$[0,1]$ with common mean $v(i)$. Therefore, Hoeffding's inequality gives
\[
\Pr\!\left(
\left|\bar V_t(i)-v(i)\right|
>
\frac{\Delta_0-\Delta}{2}
\,\middle|\, M_t
\right)
\le
2\exp\!\left(
-\frac{M_t(\Delta_0-\Delta)^2}{2}
\right).
\]
Whenever $M_t\ge N_{\mathrm{stab}}$, the right-hand side is at most
$\delta/(2KT)$. A union bound over all $t\in[T]$ and $i\in[K]$ therefore
gives, with probability at least $1-\delta/2$,
\[
\max_{i\in[K]}
\left|\bar V_t(i)-v(i)\right|
\le
\frac{\Delta_0-\Delta}{2}
\]
at every round satisfying $M_t\ge N_{\mathrm{stab}}$.

Hence, for every such round and every $i\neq i^\dagger$,
\[
\begin{aligned}
\bar V_t(i^\dagger)-\bar V_t(i)
&\ge
v(i^\dagger)-v(i)
-\left|\bar V_t(i^\dagger)-v(i^\dagger)\right|
-\left|\bar V_t(i)-v(i)\right| \\
&\ge
\Delta_0-(\Delta_0-\Delta)
=\Delta.
\end{aligned}
\]
If $\tau_{\mathrm{pop}}\le T$, then by the definition of
$\tau_{\mathrm{pop}}$, $M_t\ge N_{\mathrm{stab}}$ for every
$t\ge\tau_{\mathrm{pop}}$. Thus the persistent-optimum condition in
Assumption~\ref{ass:persistent-global-optimum} holds from $\tau_{\mathrm{pop}}$ onward,
and we may take
$\tau_{\mathrm{stab}}\le\tau_{\mathrm{pop}}$.
\end{proof}

\paragraph{Global-UCB confidence.}

Define
\[
\mathcal E_{\mathrm{UCB}}
=
\left\{
|\widehat V_t(i)-V_t(i)|
\le
B_t^{\mathrm{stat}}(i)+E_t^A,
\quad
\forall t\in[T],\ i\in[K]
\right\}.
\]

\begin{lemma}[Simultaneous confidence for Global-UCB]
\label{lem:stable-global-confidence}
Suppose rewards lie in \([0,1]\) and the entry estimates satisfy
\(
\max_{i\in[K]}
\left|
\hat\mu_{m,i}^{\mathrm{entry}}-\mu_m^i
\right|
\le
\bar P_m
\)
for every agent \(m\). If
Algorithm~\ref{alg:general-open-mamab-global} is run with confidence parameter
\(\delta/2\), then
\[
\mathbb P(\mathcal E_{\mathrm{UCB}})
\ge
1-\frac{\delta}{2}.
\]
\end{lemma}

\begin{proof}
When Algorithm~\ref{alg:general-open-mamab-global} is run with confidence
parameter \(\delta/2\), its anytime radius is
\[
c_m(n)
=
\sqrt{
\frac{1}{2n}
\log\left(
\frac{4Kj(m)(j(m)+1)n(n+1)}{\delta}
\right)
},
\qquad
n\ge1.
\]
For the \(j(m)\)-th agent, arm \(i\), and sample size \(n\ge1\), Hoeffding's
inequality gives
\[
\begin{aligned}
\mathbb P\left(
|\bar r_{m,i}(n)-\mu_m^i|>c_m(n)
\right)
&\le
2\exp\{-2n c_m(n)^2\}\\
&=
\frac{\delta}
{2Kj(m)(j(m)+1)n(n+1)}.
\end{aligned}
\]
Since
\[
\sum_{j=1}^{\infty}\frac{1}{j(j+1)}
=
\sum_{n=1}^{\infty}\frac{1}{n(n+1)}
=
1,
\]
a union bound over all agents, arms, and sample sizes yields
\[
\mathbb P\left(
|\bar r_{m,i}(n)-\mu_m^i|
\le
c_m(n),
\quad
\forall m,\ i,\ n\ge1
\right)
\ge
1-\frac{\delta}{2}.
\]

On this event, for every continuing agent \(m\in\mathcal R_t\),
\[
\left|
\hat\mu_{m,i}(t-1)-\mu_m^i
\right|
\le
\rho_{m,i}(t-1)
=
\min\left\{
\bar P_m,
c_m(n_{m,i}(t-1))
\right\},
\]
whereas for every arriving agent \(m\in\mathcal A_t\),
\[
\left|
\hat\mu_{m,i}^{\mathrm{entry}}-\mu_m^i
\right|
\le
\bar P_m.
\]
Therefore,
\[
\begin{aligned}
|\widehat V_t(i)-V_t(i)|
&\le
\sum_{m\in\mathcal R_t}
\rho_{m,i}(t-1)
+
\sum_{m\in\mathcal A_t}
\bar P_m\\
&=
B_t^{\mathrm{stat}}(i)+E_t^A
\end{aligned}
\]
simultaneously for every \(t\in[T]\) and \(i\in[K]\). Hence
\[
\mathbb P(\mathcal E_{\mathrm{UCB}})
\ge
1-\frac{\delta}{2}.
\]
\end{proof}

\paragraph{Identification by Global-UCB.}

Recall
\[
W_t(i)
=
B_t^{\mathrm{stat}}(i)+E_t^A
\]
and define
\[
\omega_t(i)
=
\frac{W_t(i)}{M_t}.
\]
The identification time is
\[
\tau_{\mathrm{id}}
=
\inf\left\{
t\in[T]:
\max_{i\in[K]}\omega_s(i)
\le
\frac{\Delta}{4},
\quad
\forall s\in\{t,\ldots,T\}
\right\},
\]
with \(\inf\emptyset=T+1\). Thus, once \(t\ge\tau_{\mathrm{id}}\),
\[
B_t^{\mathrm{stat}}(i)+E_t^A
\le
\frac{M_t\Delta}{4},
\qquad
\forall i\in[K].
\]
This definition is pathwise and does not require an additional probabilistic
event.

We next give a sufficient upper bound on \(\tau_{\mathrm{id}}\).

\begin{Proposition}[Closed-form upper bound on \(\tau_{\mathrm{id}}\)]
\label{prop:tau-id-closed-form}
Suppose there are no departures, \(M_t\ge \underline\lambda t\), 
\(|\mathcal A_t|\le \bar A\), and 
\(\bar P_m\le C_P t^{-\alpha}\) for every \(t\in[T]\) and 
\(m\in\mathcal A_t\), where 
\(\underline\lambda,\bar A,C_P>0\) and \(\alpha\ge0\). Let 
\(P_0=\sum_{m\in\mathcal M_0}\bar P_m\).

If \(\alpha=0\) and 
\(\bar A C_P/\underline\lambda<\Delta/4\), then
\(\tau_{\mathrm{id}}
\le
\max\!\left\{1,
\left\lceil
P_0/
\bigl[\underline\lambda(\Delta/4-\bar A C_P/\underline\lambda)\bigr]
\right\rceil\right\}\).

If \(0<\alpha<1\), then
\(\tau_{\mathrm{id}}
\le
\left\lceil
\max\!\left\{
8P_0/(\underline\lambda\Delta),
\left[8\bar A C_P/
(\underline\lambda(1-\alpha)\Delta)\right]^{1/\alpha}
\right\}
\right\rceil\).

If \(\alpha=1\), let 
\(A_\Delta=\max\{1,8\bar A C_P/(\underline\lambda\Delta)\}\). Then
\(\tau_{\mathrm{id}}
\le
\left\lceil
\max\!\left\{
8P_0/(\underline\lambda\Delta),
2A_\Delta\log(2eA_\Delta)
\right\}
\right\rceil\).

If \(\alpha>1\), then
\(\tau_{\mathrm{id}}
\le
\left\lceil
4\bigl(P_0+\frac{\alpha}{\alpha-1}\bar A C_P\bigr)/
(\underline\lambda\Delta)
\right\rceil\).
\end{Proposition}

\begin{proof}
Since \(\rho_{m,i}(t-1)\le \bar P_m\), for every arm \(i\),
\[
W_t(i)
\le
\sum_{m\in\mathcal M_t}\bar P_m
\le
P_0+\bar A C_P H_\alpha(t),
\qquad
H_\alpha(t)=\sum_{s=1}^t s^{-\alpha}.
\]
Hence, using \(M_t\ge\underline\lambda t\),
\begin{equation}
\max_{i\in[K]}\omega_t(i)
\le
\frac{P_0+\bar A C_P H_\alpha(t)}
{\underline\lambda t}.
\tag{A}
\end{equation}

If \(\alpha=0\), then \(H_0(t)=t\), so
\[
\max_i\omega_t(i)
\le
\frac{P_0}{\underline\lambda t}
+
\frac{\bar A C_P}{\underline\lambda}.
\]
Under \(\bar A C_P/\underline\lambda<\Delta/4\), the right-hand side is
at most \(\Delta/4\) whenever
\[
t\ge
\frac{P_0}
{\underline\lambda
(\Delta/4-\bar A C_P/\underline\lambda)}.
\]
It is nonincreasing in \(t\), yielding the first claim; the additional
maximum with \(1\) covers \(P_0=0\).

If \(0<\alpha<1\), then
\(H_\alpha(t)\le t^{1-\alpha}/(1-\alpha)\). Thus, by~(A),
\[
\max_i\omega_t(i)
\le
\frac{P_0}{\underline\lambda t}
+
\frac{\bar A C_P}
{\underline\lambda(1-\alpha)}t^{-\alpha}.
\]
The stated threshold makes each term at most \(\Delta/8\), and both
terms are nonincreasing in \(t\), proving the second claim.

If \(\alpha=1\), then \(H_1(t)\le\log(et)\), and hence
\[
\max_i\omega_t(i)
\le
\frac{P_0}{\underline\lambda t}
+
\frac{\bar A C_P}{\underline\lambda}
\frac{\log(et)}{t}.
\]
The first term is at most \(\Delta/8\) when
\(t\ge 8P_0/(\underline\lambda\Delta)\). Moreover,
\(\bar A C_P/\underline\lambda\le A_\Delta\Delta/8\).
Let \(t_0=2A_\Delta\log(2eA_\Delta)\). Since
\(B=\log(2eA_\Delta)>1\),
\[
A_\Delta\log(et_0)
=
A_\Delta\bigl(B+\log B\bigr)
\le
2A_\Delta B
=
t_0.
\]
Because \(\log(et)/t\) is nonincreasing for \(t\ge1\), the same inequality
holds for every \(t\ge t_0\). Thus the second term is also at most
\(\Delta/8\), proving the third claim.

Finally, if \(\alpha>1\), then
\[
H_\alpha(t)
\le
1+\int_1^\infty x^{-\alpha}\,dx
=
\frac{\alpha}{\alpha-1}.
\]
Therefore,
\[
\max_i\omega_t(i)
\le
\frac{
P_0+\frac{\alpha}{\alpha-1}\bar A C_P
}{
\underline\lambda t
},
\]
which is at most \(\Delta/4\) under the stated threshold and remains so
thereafter. This proves the final claim.
\end{proof}

\paragraph{Closed-form burn-in bound.}

We now combine the two components under the stationary population model.
Under Assumption~\ref{ass:stationary-arrival-composition} and the conditions of
Proposition~\ref{prop:tau-id-closed-form},
\[
\tau_{\mathrm{pop}}
\le
\left\lceil
\frac{N_{\mathrm{stab}}}{\underline\lambda}
\right\rceil.
\]
Let
\[
\tau_{\mathrm{id}}^{\mathrm{ub}}
\]
denote the corresponding upper bound in
Proposition~\ref{prop:tau-id-closed-form}, and define
\[
\tau_{\mathrm{cf}}
=
\max\left\{
\left\lceil
\frac{N_{\mathrm{stab}}}{\underline\lambda}
\right\rceil,
\tau_{\mathrm{id}}^{\mathrm{ub}}
\right\}.
\]
If $\tau_{\mathrm{cf}}\le T$, then
$\tau_{\mathrm{pop}}\le T$, and Lemma~\ref{lem:population-stability}
implies that, with probability at least \(1-\delta/2\),
\[
\tau_{\mathrm{stab}}
\le
\tau_{\mathrm{pop}}
\le
\left\lceil
\frac{N_{\mathrm{stab}}}{\underline\lambda}
\right\rceil.
\]
Therefore, on the same event,
\[
\tau
=
\max\{\tau_{\mathrm{stab}},\tau_{\mathrm{id}}\}
\le
\tau_{\mathrm{cf}}.
\]
If $\tau_{\mathrm{cf}}>T$, then
$\min\{T,\tau_{\mathrm{cf}}\}=T$, and the corresponding finite-horizon
regret bounds below follow directly from the trivial bounds over the first
\(T\) rounds.

\begin{Corollary}[Explicit stable-regime regret bounds]
\label{cor:stable-regret-closed-form}
Suppose Assumption~\ref{ass:stationary-arrival-composition} and the conditions
of Proposition~\ref{prop:tau-id-closed-form} hold, and run
Algorithm~\ref{alg:general-open-mamab-global} with confidence parameter
\(\delta/2\). Then, with probability at least \(1-\delta\),
\[
\bar R_T
\le
\min\{T,\tau_{\mathrm{cf}}\}.
\]
If, in addition,
\(
M_t\le M_0+\bar A t,
\qquad
\forall t\in[T],
\)
and
\(
q_T=\min\{T,\tau_{\mathrm{cf}}\},
\)
then
\[
R_T
\le
M_0q_T
+
\frac{\bar A q_T(q_T+1)}{2}.
\]
\end{Corollary}

\begin{proof}
Let
\[
q_T=\min\{T,\tau_{\mathrm{cf}}\}.
\]

We consider two cases. First, suppose $\tau_{\mathrm{cf}}\le T$.
Since
\[
\tau_{\mathrm{pop}}
\le
\left\lceil \frac{N_{\mathrm{stab}}}{\underline\lambda}\right\rceil.
\le
\tau_{\mathrm{cf}}
\le T,
\]
Lemma~\ref{lem:population-stability} implies, with probability at least
$1-\delta/2$, that Assumption~\ref{ass:persistent-global-optimum} holds with
\[
\tau_{\mathrm{stab}}
\le
\tau_{\mathrm{pop}}
\le
\left\lceil \frac{N_{\mathrm{stab}}}{\underline\lambda}\right\rceil..
\]
By Lemma~\ref{lem:stable-global-confidence}, the simultaneous
Global-UCB confidence event holds with probability at least
$1-\delta/2$. Hence, by a union bound, both events hold simultaneously
with probability at least $1-\delta$. On this event,
\[
\tau
=
\max\{\tau_{\mathrm{stab}},\tau_{\mathrm{id}}\}
\le
\tau_{\mathrm{cf}}.
\]
Theorem~\ref{thm:stable-global-arm} therefore gives
\[
\bar R_T\le \tau_{\mathrm{cf}}-1\le q_T.
\]
Similarly,
\[
R_T
\le
\sum_{t=1}^{\tau_{\mathrm{cf}}-1}M_t
\le
\sum_{t=1}^{q_T}M_t.
\]

Now suppose $\tau_{\mathrm{cf}}>T$. Then $q_T=T$. Since each
population-average reward lies in $[0,1]$,
\[
\bar R_T\le T=q_T,
\]
and likewise
\[
R_T\le \sum_{t=1}^T M_t
=
\sum_{t=1}^{q_T}M_t.
\]

Thus, in both cases,
\[
\bar R_T\le q_T=\min\{T,\tau_{\mathrm{cf}}\},
\qquad
R_T\le \sum_{t=1}^{q_T}M_t.
\]
If, in addition, $M_t\le M_0+\bar A t$ for all $t\in[T]$, then
\[
R_T
\le
\sum_{t=1}^{q_T}(M_0+\bar A t)
=
M_0q_T+\frac{\bar A q_T(q_T+1)}{2}.
\]
\end{proof}

\paragraph{Rate implications.}

The closed-form bound separates the time needed for population stability from
the time needed for identification. Under the stationary population model,
\[
\frac{N_{\mathrm{stab}}}{\underline\lambda}
=
O\left(
\frac{1}
{\underline\lambda(\Delta_0-\Delta)^2}
\log\frac{KT}{\delta}
\right).
\]
In contrast, the bounds on \(\tau_{\mathrm{id}}\) in
Proposition~\ref{prop:tau-id-closed-form} do not grow with \(T\) when the
remaining parameters are independent of \(T\). For example, when
\(\bar P_m=O(t^{-1/2})\) for agents arriving at round \(t\),
\[
\tau_{\mathrm{id}}
=
O(\Delta^{-2}),
\]
up to the population-growth and entry-error constants, while
\(\bar P_m=O(t^{-1})\) yields
\[
\tau_{\mathrm{id}}
=
\widetilde O(\Delta^{-1}).
\]

Consequently, if
\(K,\delta,\underline\lambda,\Delta_0,\Delta,\bar A,C_P,\alpha\), and
\(P_0\) are independent of \(T\), with
\[
\frac{\bar A C_P}{\underline\lambda}<\frac{\Delta}{4}
\]
when \(\alpha=0\), then
\[
\tau_{\mathrm{cf}}
=
O(\log T).
\]
If also
\[
M_t\le M_0+\bar A t,
\]
then
\[
\bar R_T
=
O(\log T),
\qquad
R_T
=
O\left(
M_0\log T+\bar A\log^2 T
\right).
\]
Thus, under this sufficient population model, the population may continue to
grow throughout the horizon, while regret accumulates only during a
logarithmic burn-in period. More generally, Theorem~\ref{thm:stable-global-arm}
does not require the stationary population model: any instance satisfying
Assumption~\ref{ass:persistent-global-optimum} admits the same
post-\(\tau\) identification guarantee, with the resulting regret determined
by its instance-dependent values of \(\tau_{\mathrm{stab}}\) and
\(\tau_{\mathrm{id}}\).

\subsection{Computation Complexity}\label{app:complexity}

The manifold methods inherit the computational cost of exact kernel
regression. Recall that \(N_i(t-1)\) and \(N_{t-1}\) are the numbers of
observations of arm \(i\) and all arms, respectively, by \(t-1\). For each
arm \(i\), the vector \(q_{i,t}\) has \(N_i(t-1)\) entries, each involving
\(M_t\) kernel evaluations. Constructing \(q_{i,t}\) for all arms therefore
takes \(O(N_{t-1}M_t)\) time. Similarly, \(c_t\) contains \(M_t^2\)
pairwise kernel evaluations and takes \(O(M_t^2)\) time.

To evaluate \(\widehat V_t(i)\) and
\(\sigma_{i,t}^{\mathrm{pop}}\), we solve systems involving
\(K_{i,t-1}+I\). Reusing the matrix decomposition maintained from previous
observations, each solve takes \(O(N_i(t-1)^2)\) time, and hence the
corresponding cost over all arms is
\(O(\sum_{i=1}^K N_i(t-1)^2)\). Altogether, evaluating the UCB indices takes
\(O(N_{t-1}M_t+M_t^2+\sum_{i=1}^K N_i(t-1)^2)\) time.

After the common arm \(a_t\) is selected, \(M_t\) new observations are added
to its kernel matrix. Updating the matrix decomposition once for each new
observation costs at most
\(O((N_{a_t}(t-1)+M_t)^2)\) per observation, and therefore
\(O(M_t(N_{a_t}(t-1)+M_t)^2)\) for the batch. Combining these computations,
the overall per-round complexity is \(O(M_t(N_{t-1}+M_t)^2)\) in the worst
case, which reduces to \(O(N_{t-1}^2)\) when the population size is uniformly
bounded.

Parallel exploration incurs the same cost from the UCB-style evaluation, together with the
costs of evaluating the assignment gains and solving the matching problem.
For a candidate pair \((m,i)\), directly evaluating \(G_{i,t}(m)\) requires
solving a kernel system of dimension \(N_i(t-1)\), which takes
\(O(N_i(t-1)^2)\) time using the maintained matrix decomposition. There are
\(M_t\) candidate agents for each \(i\in\mathcal C_t\), giving
\(O(M_t\sum_{i\in\mathcal C_t}N_i(t-1)^2)\) time for all assignment gains.

The matching can be formulated as a minimum-cost flow problem on a bipartite
network with \(M_t+|\mathcal C_t|+2\) vertices and
\(M_t|\mathcal C_t|+M_t+|\mathcal C_t|\) edges, where \(e_t\) units of flow
correspond to the selected agent--arm pairs. Using a network-flow algorithm, the matching can be solved in
\(O(e_t M_t|\mathcal C_t|\log(M_t+|\mathcal C_t|))\) time, which is at most
\(O(M_t|\mathcal C_t|\min\{M_t,|\mathcal C_t|\}
\log(M_t+|\mathcal C_t|))\). 

Since
\(M_t\sum_{i\in\mathcal C_t}N_i(t-1)^2
\le M_tN_{t-1}^2\)
and
\(\sum_{i=1}^K M_{t,i}(N_i(t-1)+M_{t,i})^2
\le M_t(N_{t-1}+M_t)^2\),
the overall per-round complexity is
\(O(M_t(N_{t-1}+M_t)^2+
M_t|\mathcal C_t|\min\{M_t,|\mathcal C_t|\}\log(M_t+|\mathcal C_t|))\).
For fixed \(K\), the dominating term is still $O(M_t(N_{t-1}+M_t)^2)$, which is consistent with the homogeneous exploration case.  For a
uniformly bounded population size the complexity again reduces to
\(O(N_{t-1}^2)\).

\subsection{Proof of Results in General Systems}

\subsubsection{Proof of Theorem \ref{thm:pretrained_general}}

\begin{proof}
Fix an agent \(m\), an arm \(i\), and a sample size \(n\ge1\). By
Hoeffding's inequality and the definition
\[
c_m(n)
=
\sqrt{
\frac{1}{2n}
\log\!\left(
\frac{2Kj(m)(j(m)+1)n(n+1)}{\delta}
\right)
},
\]
we have
\[
\Pr\!\left(
\left|\bar r_{m,i}(n)-\mu_m^i\right|>c_m(n)
\right)
\le
2\exp\{-2n c_m(n)^2\}
=
\frac{\delta}
{Kj(m)(j(m)+1)n(n+1)}.
\]
Since the agents are indexed by their order of first appearance and
\[
\sum_{j=1}^{\infty}\frac{1}{j(j+1)}
=
\sum_{n=1}^{\infty}\frac{1}{n(n+1)}
=
1,
\]
a union bound over all agents, arms, and sample sizes gives, with probability
at least \(1-\delta\),
\[
\left|\bar r_{m,i}(n)-\mu_m^i\right|
\le
c_m(n)
\]
simultaneously for every agent \(m\), arm \(i\), and \(n\ge1\).
We work on this event throughout the remainder of the proof.

Together with the entry guarantee
\[
\max_{i\in[K]}
\left|
\hat\mu_{m,i}^{\mathrm{entry}}-\mu_m^i
\right|
\le
\bar P_m,
\]
this implies, for every continuing agent \(m\in\mathcal R_t\),
\[
\left|
\hat\mu_{m,i}(t-1)-\mu_m^i
\right|
\le
\rho_{m,i}(t-1)
=
\min\left\{
\bar P_m,\,
c_m(n_{m,i}(t-1))
\right\}.
\]
For every newly arriving agent \(m\in\mathcal A_t\),
\[
\left|
\hat\mu_{m,i}^{\mathrm{entry}}-\mu_m^i
\right|
\le
\bar P_m.
\]
Therefore, for every arm \(i\),
\[
\begin{aligned}
\left|
\widehat V_t(i)-V_t(i)
\right|
&\le
\sum_{m\in\mathcal R_t}
\rho_{m,i}(t-1)
+
\sum_{m\in\mathcal A_t}\bar P_m\\
&=
B_t^{\mathrm{stat}}(i)+E_t^A.
\end{aligned}
\]

Let
\[
i_t^\star\in\arg\max_{i\in[K]}V_t(i).
\]
Since \(E_t^A\) is common across arms, the UCB rule gives
\[
\widehat V_t(i_t^\star)
+
B_t^{\mathrm{stat}}(i_t^\star)
\le
\widehat V_t(a_t)
+
B_t^{\mathrm{stat}}(a_t).
\]
Using the confidence bounds on both \(i_t^\star\) and \(a_t\),
\[
\begin{aligned}
V_t(i_t^\star)
&\le
\widehat V_t(i_t^\star)
+
B_t^{\mathrm{stat}}(i_t^\star)
+
E_t^A\\
&\le
\widehat V_t(a_t)
+
B_t^{\mathrm{stat}}(a_t)
+
E_t^A\\
&\le
V_t(a_t)
+
2B_t^{\mathrm{stat}}(a_t)
+
2E_t^A.
\end{aligned}
\]
Hence
\[
V_t(i_t^\star)-V_t(a_t)
\le
2B_t^{\mathrm{stat}}(a_t)+2E_t^A.
\]
Because all active agents select the common arm \(a_t\),
\[
R_T
=
\sum_{t=1}^T
\left[
V_t(i_t^\star)-V_t(a_t)
\right],
\]
and therefore
\[
R_T
\le
2\sum_{t=1}^T B_t^{\mathrm{stat}}(a_t)
+
2\sum_{t=1}^T E_t^A.
\]

It remains to control the first term. By the definition of
\(B_t^{\mathrm{stat}}\),
\[
\begin{aligned}
\sum_{t=1}^T B_t^{\mathrm{stat}}(a_t)
&=
\sum_{t=1}^T
\sum_{m\in\mathcal R_t}
\rho_{m,a_t}(t-1)\\
&=
\sum_{m:L_m>0}
\sum_{i=1}^K
\sum_{\substack{t\le T:\\
m\in\mathcal R_t,\ a_t=i}}
\rho_{m,i}(t-1).
\end{aligned}
\]
Thus, the outer sum is over distinct agents, while \(L_m\) records the total
number of rounds in which agent \(m\) contributes to this term.

Fix an agent \(m\) with \(L_m>0\) and an arm \(i\), and write
\[
n=N_{m,i}^{R}(T).
\]
Since
\(\rho_{m,i}(t-1)\le\bar P_m\),
\[
\sum_{\substack{t\le T:\\
m\in\mathcal R_t,\ a_t=i}}
\rho_{m,i}(t-1)
\le
n\bar P_m.
\]

We next derive an alternative statistical bound. Let
\[
t_1<t_2<\cdots<t_n
\]
be the rounds at which \(m\in\mathcal R_{t_r}\) and \(a_{t_r}=i\).
The first such round contributes at most \(\bar P_m\). Before round
\(t_r\), \(r\ge2\), agent \(m\) has already observed arm \(i\) on at least
\(r-1\) previous rounds. Moreover, \(c_m(\cdot)\) is nonincreasing on the
positive integers. Indeed, writing
\[
C_m=\frac{2Kj(m)(j(m)+1)}{\delta},
\]
the function
\[
x\mapsto
\frac{\log(C_mx(x+1))}{x}
\]
is decreasing for \(x\ge1\), since \(C_m\ge4\) and
\[
\frac{d}{dx}
\frac{\log(C_mx(x+1))}{x}
=
\frac{
1+\frac{x}{x+1}
-\log(C_mx(x+1))
}{x^2}
<0.
\]
Therefore,
\[
\rho_{m,i}(t_r-1)
\le
c_m(r-1),
\qquad r=2,\ldots,n.
\]

For \(\ell\ge1\), recall
\[
\Lambda_m(\ell)
=
\log\!\left(
\frac{2Kj(m)(j(m)+1)\ell(\ell+1)}{\delta}
\right).
\]
Since \(r-1\le n-1\le L_m\),
\[
c_m(r-1)
\le
\sqrt{
\frac{\Lambda_m(L_m)}
{2(r-1)}
}.
\]
Consequently,
\[
\begin{aligned}
\sum_{\substack{t\le T:\\
m\in\mathcal R_t,\ a_t=i}}
\rho_{m,i}(t-1)
&\le
\bar P_m
+
\sqrt{\frac{\Lambda_m(L_m)}{2}}
\sum_{s=1}^{n-1}\frac{1}{\sqrt{s}}\\
&\le
\bar P_m
+
\sqrt{2\Lambda_m(L_m)n},
\end{aligned}
\]
where we used
\[
\sum_{s=1}^{n-1}\frac{1}{\sqrt{s}}
\le
2\sqrt{n}.
\]
Thus, for every \(m\) with \(L_m>0\) and every arm \(i\),
\[
\sum_{\substack{t\le T:\\
m\in\mathcal R_t,\ a_t=i}}
\rho_{m,i}(t-1)
\le
\min\left\{
N_{m,i}^{R}(T)\bar P_m,\;
\bar P_m+
\sqrt{
2\Lambda_m(L_m)N_{m,i}^{R}(T)
}
\right\}.
\]

We now sum over arms. Since
\[
\sum_{i=1}^K N_{m,i}^{R}(T)=L_m,
\]
the first bound gives
\[
\sum_{i=1}^K
\sum_{\substack{t\le T:\\
m\in\mathcal R_t,\ a_t=i}}
\rho_{m,i}(t-1)
\le
L_m\bar P_m.
\]
Alternatively,
\[
\begin{aligned}
\sum_{i=1}^K
\sum_{\substack{t\le T:\\
m\in\mathcal R_t,\ a_t=i}}
\rho_{m,i}(t-1)
&\le
K\bar P_m
+
\sqrt{2\Lambda_m(L_m)}
\sum_{i=1}^K
\sqrt{N_{m,i}^{R}(T)}\\
&\le
K\bar P_m
+
\sqrt{2K\Lambda_m(L_m)L_m},
\end{aligned}
\]
where the last inequality follows from Cauchy--Schwarz. Taking the smaller of
the two bounds yields
\[
\sum_{i=1}^K
\sum_{\substack{t\le T:\\
m\in\mathcal R_t,\ a_t=i}}
\rho_{m,i}(t-1)
\le
\min\left\{
L_m\bar P_m,\;
K\bar P_m+
\sqrt{2K\Lambda_m(L_m)L_m}
\right\}.
\]
Summing over all agents with \(L_m>0\),
\[
\sum_{t=1}^T B_t^{\mathrm{stat}}(a_t)
\le
\sum_{m:L_m>0}
\min\left\{
L_m\bar P_m,\;
K\bar P_m+
\sqrt{2K\Lambda_m(L_m)L_m}
\right\}.
\]

Finally,
\[
\sum_{t=1}^T E_t^A
=
\sum_{t=1}^T
\sum_{m\in\mathcal A_t}
\bar P_m.
\]
Substituting the last two bounds into the regret decomposition gives
\[
R_T
\le
2\sum_{t=1}^T
\sum_{m\in\mathcal A_t}\bar P_m
+
2\sum_{m:L_m>0}
\min\left\{
L_m\bar P_m,\;
K\bar P_m+
\sqrt{2K\Lambda_m(L_m)L_m}
\right\},
\]
which proves the claim.
\end{proof}

\subsubsection{Proof of Theorem~\ref{thm:pivotal-arrival-lower-bound}}

\begin{proof}
We prove the two lower bounds separately.

\paragraph{Lower bound for newly arriving agents.}
Introduce one permanent agent \(m_0\) whose rewards are identically zero on
both arms and are known to the learner. At round \(t\), let
\[
\mathcal M_t=\{m_0\}\cup\mathcal A_t.
\]
Every agent in \(\mathcal A_t\) departs before round \(t+1\). Thus, each such
agent is newly arriving at round \(t\) and is never a continuing agent.

For every sign vector
\[
z=(z_1,\ldots,z_T)\in\{-1,+1\}^T,
\]
construct an instance \(\nu^z\) as follows. For every
\(m\in\mathcal A_t\), set
\[
(\mu_m^1,\mu_m^2)
=
\begin{cases}
(p_m,0), & z_t=+1,\\
(0,p_m), & z_t=-1.
\end{cases}
\]
Rewards are deterministic at these means. Since \(p_m\in[0,1]\), all reward
distributions are supported on \([0,1]\).

For every arriving agent, provide the same entry information under the two
configurations:
\[
\hat\mu_{m,1}^{\mathrm{entry}}
=
\hat\mu_{m,2}^{\mathrm{entry}}
=
0,
\qquad
\bar P_m=p_m.
\]
The entry estimates lie in \([0,1]\), and under either value of \(z_t\),
\[
P_m
=
\max_{i\in\{1,2\}}
\left|
\hat\mu_{m,i}^{\mathrm{entry}}-\mu_m^i
\right|
=
p_m
=
\bar P_m.
\]
Thus, the entry information available to the learner is identical under
\(z_t=+1\) and \(z_t=-1\).

Let
\[
S_t=\sum_{m\in\mathcal A_t}p_m.
\]
Because the permanent agent is arm-indifferent, if \(z_t=+1\), then
\[
V_t(1)=S_t,
\qquad
V_t(2)=0,
\]
whereas if \(z_t=-1\), then
\[
V_t(1)=0,
\qquad
V_t(2)=S_t.
\]
Hence, whenever \(S_t>0\), the optimal arm is \(1\) under \(z_t=+1\) and
\(2\) under \(z_t=-1\), and selecting the other arm incurs instantaneous
regret exactly \(S_t\). When \(S_t=0\), the instantaneous regret is zero,
so the same calculation below remains valid.

Now let \(Z_1,\ldots,Z_T\) be independent Rademacher random variables,
i.e.,
\[
\mathbb P(Z_t=+1)
=
\mathbb P(Z_t=-1)
=
\frac12.
\]
This induces the uniform prior over the \(2^T\) instances
\(\{\nu^z:z\in\{-1,+1\}^T\}\).

Fix any possibly randomized nonanticipating policy \(\pi\). Let
\(\mathcal F_t^{-}\) denote the sigma-field containing all information
available immediately before the policy selects \(a_t\), including the
complete past history, the current active population, and the entry
information of all agents in \(\mathcal A_t\).

By construction, the information in \(\mathcal F_t^{-}\) is identical
under \(Z_t=+1\) and \(Z_t=-1\). Moreover, \(Z_t\) is independent of
\(Z_1,\ldots,Z_{t-1}\). Therefore,
\[
Z_t\perp\!\!\!\perp\mathcal F_t^{-}.
\]
Let
\[
q_t
=
\mathbb P_\pi(a_t=1\mid\mathcal F_t^{-}).
\]
Conditional on \(\mathcal F_t^{-}\), when \(Z_t=+1\), the expected
instantaneous regret is
\[
(1-q_t)S_t,
\]
whereas when \(Z_t=-1\), it is
\[
q_tS_t.
\]
Averaging over \(Z_t\) gives
\[
\begin{aligned}
\mathbb E\!\left[
V_t(i_t^\star)-V_t(a_t)
\mid\mathcal F_t^{-}
\right]
&=
\frac12(1-q_t)S_t
+
\frac12q_tS_t =
\frac{S_t}{2}.
\end{aligned}
\]
Taking expectations and summing over \(t\),
\[
\mathbb E_{Z,\pi}[R_T]
=
\frac12\sum_{t=1}^TS_t
=
\frac12
\sum_{t=1}^T
\sum_{m\in\mathcal A_t}p_m.
\]
The left-hand side is the average of
\(\mathbb E_{\nu^z}^{\pi}[R_T]\) over the \(2^T\) instances \(\nu^z\).
Hence, there exists at least one \(z\in\{-1,+1\}^T\) such that
\[
\mathbb E_{\nu^z}^{\pi}[R_T]
\ge
\frac12
\sum_{t=1}^T
\sum_{m\in\mathcal A_t}p_m.
\]
Consequently,
\[
\sup_{\nu\in\mathcal I(\{p_m\})}
\mathbb E_\nu^\pi[R_T]
\ge
\frac12
\sum_{t=1}^T
\sum_{m\in\mathcal A_t}p_m
=
\frac12
\sum_{t=1}^T
\sum_{m\in\mathcal A_t}\bar P_m.
\]

\paragraph{Lower bound for post-entry learning.}
We next construct a different family of instances. Consider any collection
\[
\{(L_m,p_m):L_m>0\},
\qquad
L_m\in\mathbb N,
\qquad
p_m\in[0,1],
\]
and define
\[
\mathcal H
=
\{m:L_m>0,\ p_m>0\}.
\]
Agents with \(p_m=0\) contribute zero to the desired lower bound and may be
ignored in the hard construction. Suppose
\[
\sum_{m\in\mathcal H}(L_m+1)\le T.
\]
For every \(m\in\mathcal H\), define
\[
H_m=L_m+1.
\]
Fix any ordering of the agents in \(\mathcal H\), and assign them pairwise
disjoint consecutive blocks
\[
\mathcal B_m
=
\{s_m,s_m+1,\ldots,s_m+H_m-1\}
\subseteq[T]
\]
with
\[
|\mathcal B_m|=H_m.
\]

Agent \(m\) arrives at the first round of \(\mathcal B_m\), remains active
throughout the block, and departs immediately afterward. Thus, it is newly
arriving during the first round of \(\mathcal B_m\) and is a continuing
agent during exactly \(L_m\) subsequent rounds.

For each \(m\in\mathcal H\), introduce an auxiliary agent \(b_m\) that is
active over the same block. Its reward means are known exactly to the
learner and its pre-training error is zero. We also retain the permanent
agent \(m_0\), whose rewards are identically zero on both arms. Hence, for
\(t\in\mathcal B_m\),
\[
\mathcal M_t=\{m_0,m,b_m\},
\]
whereas at rounds outside the blocks,
\[
\mathcal M_t=\{m_0\}.
\]
The latter rounds incur zero regret.

The auxiliary agents and \(m_0\) have pre-training error zero and therefore
contribute zero to
\[
\min\{L_mp_m,\sqrt{KL_m}\}.
\]
Thus, adding them does not affect the right-hand side of the desired lower
bound.

Recall that \(K=2\). For every \(m\in\mathcal H\), define
\[
\varepsilon_m
=
\min\left\{
p_m,\,
\frac1{16}\sqrt{\frac{K}{H_m}}
\right\}.
\]
Since \(K=2\) and \(H_m\ge1\),
\[
0<\varepsilon_m
\le
\frac{\sqrt2}{16}
<
\frac18.
\]

For \(p\in[0,1]\), define
\[
c_1(p)
=
\left(\frac12-p\right)_+,
\qquad
c_2(p)
=
\left(p-\frac12\right)_+.
\]
Then
\[
c_1(p),c_2(p)\in\left[0,\frac12\right]
\]
and
\[
p+c_1(p)
=
\frac12+c_2(p)
=
\max\left\{p,\frac12\right\}.
\]

For every collection of signs
\[
z=\{z_m\in\{-1,+1\}:m\in\mathcal H\},
\]
construct an instance \(\nu^z\) as follows. During block
\(\mathcal B_m\), the hard agent \(m\) has mean rewards
\[
\mu_m^1=p_m,
\qquad
\mu_m^2=\frac12+z_m\varepsilon_m.
\]
More precisely, arm \(1\) gives the deterministic reward \(p_m\), while
arm \(2\) gives independent Bernoulli rewards with mean
\(\frac12+z_m\varepsilon_m\). Since
\(\varepsilon_m<1/8\), these rewards are supported on \([0,1]\).

The auxiliary agent \(b_m\) has deterministic rewards with means
\[
\mu_{b_m}^1=c_1(p_m),
\qquad
\mu_{b_m}^2=c_2(p_m).
\]
The permanent agent \(m_0\) has deterministic reward zero on both arms.

For the hard agent \(m\), provide the entry estimate
\[
\hat\mu_{m,1}^{\mathrm{entry}}=0,
\qquad
\hat\mu_{m,2}^{\mathrm{entry}}=\frac12,
\qquad
\bar P_m=p_m.
\]
Both coordinates belong to \([0,1]\). Moreover,
\[
\left|
\hat\mu_{m,1}^{\mathrm{entry}}-\mu_m^1
\right|
=
p_m
\]
and
\[
\left|
\hat\mu_{m,2}^{\mathrm{entry}}-\mu_m^2
\right|
=
\varepsilon_m
\le
p_m.
\]
Hence
\[
P_m
=
\max_{a\in\{1,2\}}
\left|
\hat\mu_{m,a}^{\mathrm{entry}}-\mu_m^a
\right|
=
p_m
=
\bar P_m.
\]
Thus, the prescribed pre-training error is attained exactly, while the entry
information does not depend on \(z_m\).

For the auxiliary agent, provide its exact means as its entry estimate:
\[
\hat\mu_{b_m,1}^{\mathrm{entry}}
=
c_1(p_m),
\qquad
\hat\mu_{b_m,2}^{\mathrm{entry}}
=
c_2(p_m),
\qquad
\bar P_{b_m}=0.
\]
Therefore \(P_{b_m}=0\), and all entry estimates in the construction lie in
\([0,1]\).

Define
\[
B_m
=
p_m+c_1(p_m)
=
\frac12+c_2(p_m).
\]
For every \(t\in\mathcal B_m\), the aggregate mean rewards are
\[
V_t(1)
=
p_m+c_1(p_m)
=
B_m
\]
and
\[
V_t(2)
=
\frac12+z_m\varepsilon_m+c_2(p_m)
=
B_m+z_m\varepsilon_m.
\]
Thus, when \(z_m=+1\), arm \(2\) is uniquely optimal, whereas when
\(z_m=-1\), arm \(1\) is uniquely optimal. In either case, selecting the
suboptimal arm incurs instantaneous aggregate regret exactly
\(\varepsilon_m\).

We now lower-bound the regret accumulated over one block. Fix
\(m\in\mathcal H\), and abbreviate
\[
H=H_m,
\qquad
\varepsilon=\varepsilon_m.
\]
Condition on an arbitrary realization of the complete history before
\(\mathcal B_m\) begins. The policy may depend arbitrarily on this history.
The sign \(z_m\) will be chosen independently of that history, so the
following argument applies conditionally on every such realization.

Let \(\mathbb P_+\) and \(\mathbb P_-\) denote the conditional laws of the
transcript generated during \(\mathcal B_m\) under \(z_m=+1\) and
\(z_m=-1\), respectively. Let
\[
N_i
=
\sum_{t\in\mathcal B_m}
\mathbf 1\{a_t=i\},
\qquad
i\in\{1,2\}.
\]
Since exactly one common arm is selected at every round,
\[
N_1+N_2=H
\]
pathwise.

The two instances have identical populations, entry information, auxiliary
rewards, arm-\(1\) rewards of the hard agent, and action kernels of the
policy. They differ only in the reward distribution of agent \(m\) when
arm \(2\) is selected. Therefore, by the chain rule for relative entropy,
\[
D_{\mathrm{KL}}(\mathbb P_-\|\mathbb P_+)
=
\mathbb E_-[N_2]\,
\operatorname{kl}\!\left(
\frac12-\varepsilon,
\frac12+\varepsilon
\right),
\]
where \(\operatorname{kl}(\cdot,\cdot)\) denotes Bernoulli relative
entropy.

For \(0\le\varepsilon\le1/8\),
\[
\begin{aligned}
\operatorname{kl}\!\left(
\frac12-\varepsilon,
\frac12+\varepsilon
\right)
&=
2\varepsilon
\log\left(
\frac{1+2\varepsilon}{1-2\varepsilon}
\right)\\
&\le
12\varepsilon^2.
\end{aligned}
\]
Indeed, for \(0\le x\le1/4\),
\[
\log\left(\frac{1+x}{1-x}\right)
\le
\frac{2x}{1-x}.
\]
Taking \(x=2\varepsilon\) gives
\[
\begin{aligned}
2\varepsilon
\log\left(
\frac{1+2\varepsilon}{1-2\varepsilon}
\right)
&\le
\frac{8\varepsilon^2}{1-2\varepsilon}\\
&\le
\frac{32}{3}\varepsilon^2\\
&\le
12\varepsilon^2.
\end{aligned}
\]
Hence
\[
D_{\mathrm{KL}}(\mathbb P_-\|\mathbb P_+)
\le
12H\varepsilon^2.
\]
By the definition of \(\varepsilon\),
\[
H\varepsilon^2
\le
\frac{K}{256}
=
\frac1{128},
\]
where the equality uses \(K=2\). Therefore,
\[
D_{\mathrm{KL}}(\mathbb P_-\|\mathbb P_+)
\le
\frac{3}{32}.
\]
Pinsker's inequality then gives
\[
\|\mathbb P_--\mathbb P_+\|_{\mathrm{TV}}
\le
\sqrt{
\frac12
D_{\mathrm{KL}}(\mathbb P_-\|\mathbb P_+)
}
\le
\frac{\sqrt3}{8}
<
\frac14.
\]

Let \(R_m\) denote the regret accumulated over block
\(\mathcal B_m\). Under \(\mathbb P_+\), arm \(2\) is optimal, so every
selection of arm \(1\) incurs regret \(\varepsilon\). Thus,
\[
\mathbb E_+[R_m]
=
\varepsilon\mathbb E_+[N_1].
\]
Under \(\mathbb P_-\), arm \(1\) is optimal, so every selection of arm \(2\)
incurs regret \(\varepsilon\). Thus,
\[
\mathbb E_-[R_m]
=
\varepsilon\mathbb E_-[N_2].
\]

Since \(0\le N_1\le H\),
\[
\left|
\mathbb E_+[N_1]-\mathbb E_-[N_1]
\right|
\le
H
\|\mathbb P_+-\mathbb P_-\|_{\mathrm{TV}}.
\]
Therefore,
\[
\mathbb E_+[N_1]
\ge
\mathbb E_-[N_1]
-
H
\|\mathbb P_+-\mathbb P_-\|_{\mathrm{TV}}.
\]
Using \(N_1+N_2=H\) under \(\mathbb P_-\),
\[
\begin{aligned}
\mathbb E_+[N_1]+\mathbb E_-[N_2]
&\ge
\mathbb E_-[N_1]
+
\mathbb E_-[N_2]
-
H
\|\mathbb P_+-\mathbb P_-\|_{\mathrm{TV}}\\
&=
H
\left(
1-
\|\mathbb P_+-\mathbb P_-\|_{\mathrm{TV}}
\right)\ge
\frac{3H}{4}.
\end{aligned}
\]
Consequently,
\[
\begin{aligned}
\frac12
\left(
\mathbb E_+[R_m]
+
\mathbb E_-[R_m]
\right)
&=
\frac{\varepsilon}{2}
\left(
\mathbb E_+[N_1]
+
\mathbb E_-[N_2]
\right)\ge
\frac{3}{8}H\varepsilon.
\end{aligned}
\]

Substituting the definition of \(\varepsilon_m\),
\[
\frac{3}{8}H_m\varepsilon_m
=
\min\left\{
\frac{3}{8}H_mp_m,\,
\frac{3}{128}\sqrt{KH_m}
\right\}.
\]
Since \(H_m=L_m+1\ge L_m\),
\[
\frac{3}{8}H_mp_m
\ge
\frac{1}{128}L_mp_m
\]
and
\[
\frac{3}{128}\sqrt{KH_m}
\ge
\frac{1}{128}\sqrt{KL_m}.
\]
It follows that
\[
\frac12
\left(
\mathbb E_+[R_m]
+
\mathbb E_-[R_m]
\right)
\ge
\frac{1}{128}
\min\left\{
L_mp_m,\sqrt{KL_m}
\right\}.
\]

Finally, let
\[
\{Z_m:m\in\mathcal H\}
\]
be independent Rademacher random variables. Consider the resulting uniform
prior over the finite family of instances
\[
\left\{
\nu^z:
z=\{z_m\in\{-1,+1\}:m\in\mathcal H\}
\right\}.
\]
Before block \(\mathcal B_m\) begins, \(Z_m\) is independent of the complete
history generated by all blocks that precede \(\mathcal B_m\). Moreover,
the active population and all entry information in \(\mathcal B_m\) are
identical under the two values of \(Z_m\). Hence the preceding conditional
two-instance argument applies after conditioning on the complete history
before \(\mathcal B_m\). Taking expectations gives
\[
\mathbb E_{Z,\pi}[R_m]
\ge
\frac{1}{128}
\min\left\{
L_mp_m,\sqrt{KL_m}
\right\}
\]
for every \(m\in\mathcal H\).

Because the blocks are disjoint and all rounds outside them incur zero
regret,
\[
\begin{aligned}
\mathbb E_{Z,\pi}[R_T]
&=
\sum_{m\in\mathcal H}
\mathbb E_{Z,\pi}[R_m]\ge
\frac{1}{128}
\sum_{m\in\mathcal H}
\min\left\{
L_mp_m,\sqrt{KL_m}
\right\}.
\end{aligned}
\]
For every \(m\) with \(L_m>0\) but \(p_m=0\),
\[
\min\left\{
L_mp_m,\sqrt{KL_m}
\right\}
=
0.
\]
Therefore,
\[
\mathbb E_{Z,\pi}[R_T]
\ge
\frac{1}{128}
\sum_{m:L_m>0}
\min\left\{
L_mp_m,\sqrt{KL_m}
\right\}.
\]

The left-hand side is the average of
\(\mathbb E_{\nu^z}^{\pi}[R_T]\) over the finite family of instances
\(\{\nu^z\}\). Hence, at least one instance in the family satisfies
\[
\mathbb E_{\nu^z}^{\pi}[R_T]
\ge
\frac{1}{128}
\sum_{m:L_m>0}
\min\left\{
L_mp_m,\sqrt{KL_m}
\right\}.
\]
Consequently,
\[
\sup_{\nu\in\mathcal Q(\{L_m,p_m\})}
\mathbb E_\nu^\pi[R_T]
\ge
\frac{1}{128}
\sum_{m:L_m>0}
\min\left\{
L_mp_m,\sqrt{KL_m}
\right\}.
\]

Both constructions apply to every admissible possibly randomized policy \(\pi\), which proves the claims.
\end{proof}

\subsubsection{Proof of Theorem \ref{thm:stable-global-arm}}

\begin{proof}

By Assumption~\ref{ass:persistent-global-optimum}, for every $t\ge \tau_{\mathrm{stab}}$ and
$i\neq i^\dagger$,
\[
V_t(i^\dagger)-V_t(i)
=
M_t\bigl(\bar V_t(i^\dagger)-\bar V_t(i)\bigr)
\ge M_t\Delta.
\]

By Lemma~\ref{lem:stable-global-confidence}, with probability at least
\(1-\delta/2\),
\[
|\widehat V_t(i)-V_t(i)|
\le
B_t^{\mathrm{stat}}(i)+E_t^A,
\qquad
\forall t\in[T],\ i\in[K].
\]
A union bound shows that both statements hold simultaneously with probability
at least \(1-\delta\). We work on this event.

Fix \(t\in\{\tau,\ldots,T\}\) and \(i\neq i^\dagger\). Since
\(t\ge\tau_{\mathrm{stab}}\),
\[
V_t(i^\dagger)-V_t(i)
=
M_t\bigl(
\bar V_t(i^\dagger)-\bar V_t(i)
\bigr)
\ge
M_t\Delta.
\]
Since \(t\ge\tau_{\mathrm{id}}\),
\[
W_t(i)
\le
\frac{M_t\Delta}{4}.
\]

Let
\[
U_t(i)
=
\widehat V_t(i)+B_t^{\mathrm{stat}}(i)
\]
be the Global-UCB index. On the simultaneous confidence event,
\[
U_t(i^\dagger)
\ge
V_t(i^\dagger)-E_t^A,
\]
whereas
\[
U_t(i)
\le
V_t(i)
+
2B_t^{\mathrm{stat}}(i)
+
E_t^A.
\]
Therefore,
\[
\begin{aligned}
U_t(i^\dagger)-U_t(i)
&\ge
V_t(i^\dagger)-V_t(i)
-
2B_t^{\mathrm{stat}}(i)
-
2E_t^A\\
&=
V_t(i^\dagger)-V_t(i)-2W_t(i)\\
&\ge
M_t\Delta-\frac{M_t\Delta}{2}=
\frac{M_t\Delta}{2}
>0.
\end{aligned}
\]
Thus \(a_t=i^\dagger\). Since \(t\ge\tau\) was arbitrary, the same conclusion
holds for every subsequent round.

Under coordinated actions,
\[
\bar R_T
=
\sum_{t=1}^T
\left[
\bar V_t(i_t^\star)-\bar V_t(a_t)
\right].
\]
All terms from \(t=\tau\) onward vanish, while each preceding term is at most
one. Hence
\[
\bar R_T\le\tau-1.
\]
Similarly,
\[
R_T
=
\sum_{t=1}^T
\left[
V_t(i_t^\star)-V_t(a_t)
\right]
=
\sum_{t=1}^T
M_t
\left[
\bar V_t(i_t^\star)-\bar V_t(a_t)
\right],
\]
so
\[
R_T
\le
\sum_{t=1}^{\tau-1}M_t.
\]
\end{proof}

\subsubsection{Proof of Lemma \ref{lem:population-functional-confidence}}

\begin{proof}
Fix an arm \(i\), and order all agent-level observations collected from this
arm chronologically, using an arbitrary fixed ordering among agents observed
in the same round. Before the reward noises of a round are realized, the
population, its representations, and the selected arm are already determined.
Under the conditional independence assumption, this serialization produces a
predictable sequence of kernel features with conditionally
\(R\)-sub-Gaussian noise.

Applying the RKHS self-normalized concentration inequality with failure
probability \(\delta/K\) therefore gives, simultaneously over all sample
sizes,
\[
\left\|
\widehat f_{i,t-1}-f_i
\right\|_{\mathcal A_{i,t-1}}
\le
S+
R\sqrt{
2\left(
\gamma_{N_i(t-1)}
+
1+
\log\frac{K}{\delta}
\right)
}.
\]
Since
\(\gamma_{N_i(t-1)}\le\bar\gamma_{N_i(t-1)}\), the right-hand side is at
most \(\beta_{i,t-1}\). A union bound over the \(K\) arms gives a common event
of probability at least \(1-\delta\).

On this event,
\[
\begin{aligned}
\left|
\widehat V_t(i)-V_t(i)
\right|
&=
\left|
\left\langle
\widehat f_{i,t-1}-f_i,\Psi_t
\right\rangle_{\mathcal H_k}
\right|\\
&\le
\left\|
\widehat f_{i,t-1}-f_i
\right\|_{\mathcal A_{i,t-1}}
\left\|
\Psi_t
\right\|_{\mathcal A_{i,t-1}^{-1}}\\
&\le
\beta_{i,t-1}\sigma_{i,t}^{\mathrm{pop}},
\end{aligned}
\]
which proves the claim.
\end{proof}

\subsubsection{Proof of Theorem \ref{thm:manifold-regret}}

\begin{proof}
Work on the event in
Lemma~\ref{lem:population-functional-confidence}. Let
\(i_t^\star\in\arg\max_{i\in[K]}V_t(i)\). By the UCB rule,
\[
\widehat V_t(i_t^\star)
+
\beta_{i_t^\star,t-1}
\sigma_{i_t^\star,t}^{\mathrm{pop}}
\le
\widehat V_t(a_t)
+
\beta_{a_t,t-1}
\sigma_{a_t,t}^{\mathrm{pop}}.
\]
The confidence bounds therefore imply
\[
\begin{aligned}
V_t(i_t^\star)-V_t(a_t)
&\le
2\beta_{a_t,t-1}
\sigma_{a_t,t}^{\mathrm{pop}}\le
2\beta_T\sigma_{a_t,t}^{\mathrm{pop}}.
\end{aligned}
\]
Thus,
\[
R_T
\le
2\beta_T
\sum_{t=1}^T
\sigma_{a_t,t}^{\mathrm{pop}}.
\]

We next relate the population-functional uncertainty to the information gained
from the batch of observations obtained at round \(t\). Write
\(i=a_t\), let
\[
\mathcal B_t
=
\sum_{m\in\mathcal M_t}
\phi(x_m)\otimes\phi(x_m),
\]
and define
\[
\mathcal X_t
=
\mathcal A_{i,t-1}^{-1/2}
\mathcal B_t
\mathcal A_{i,t-1}^{-1/2}.
\]
For
\(u_m=\mathcal A_{i,t-1}^{-1/2}\phi(x_m)\),
\[
\begin{aligned}
\left(\sigma_{a_t,t}^{\mathrm{pop}}\right)^2
&=
\|
\sum_{m\in\mathcal M_t}u_m
\|_{\mathcal H_k}^2\le
M_t
\sum_{m\in\mathcal M_t}
\|u_m\|_{\mathcal H_k}^2=
M_t\operatorname{tr}(\mathcal X_t).
\end{aligned}
\]
Since \(\mathcal A_{i,t-1}\succeq I\) and \(k(x,x)\le1\),
\[
\operatorname{tr}(\mathcal X_t)
=
\sum_{m\in\mathcal M_t}\|u_m\|_{\mathcal H_k}^2
\le M_t.
\]
Hence every eigenvalue \(z\) of \(\mathcal X_t\) lies in \([0,M_t]\).
Because \(z/\log(1+z)\) is increasing for \(z\ge0\),
\[
\operatorname{tr}(\mathcal X_t)
\le
\frac{M_t}{\log(1+M_t)}
\log\det(I+\mathcal X_t).
\]
Consequently,
\[
\left(\sigma_{a_t,t}^{\mathrm{pop}}\right)^2
\le
\frac{M_t^2}{\log(1+M_t)}
\log\det(I+\mathcal X_t).
\]

Let
\[
\Delta\Gamma_t
=
\Gamma_{a_t,t}-\Gamma_{a_t,t-1}.
\]
By the determinant identity for the update
\(\mathcal A_{a_t,t}
=\mathcal A_{a_t,t-1}+\mathcal B_t\),
\[
\Delta\Gamma_t
=
\frac12\log\det(I+\mathcal X_t).
\]
Therefore,
\[
\left(\sigma_{a_t,t}^{\mathrm{pop}}\right)^2
\le
2
\frac{M_t^2}{\log(1+M_t)}
\Delta\Gamma_t.
\]
By Cauchy--Schwarz,
\[
\begin{aligned}
\sum_{t=1}^T
\sigma_{a_t,t}^{\mathrm{pop}}
&\le
\sqrt{
2
\left(
\sum_{t=1}^T
\frac{M_t^2}{\log(1+M_t)}
\right)
\left(
\sum_{t=1}^T\Delta\Gamma_t
\right)
}\\
&=
\sqrt{
2H_T\Gamma_T^\Sigma
},
\end{aligned}
\]
because the information-gain increments telescope separately for each arm.
Substitution proves the first claim.

Finally, arm \(i\) receives \(N_i(T)\) agent-level observations, so
\[
\Gamma_{i,T}
\le
\gamma_{N_i(T)}
\le
\bar\gamma_{N_i(T)}.
\]
Summing over arms and using
\(N_i(T)\le N_T\) proves the second claim.
\end{proof}

\subsubsection{Proof of Theorem \ref{thm:manifold-matern-regret}}

\begin{proof}

We first present the information gain under the kernel. 

\begin{lemma}[Information gain]
\label{lem:manifold-matern-information}
There exists a constant \(C_\gamma>0\), depending only on
\((\mathcal X,g)\), \(\nu\), and \(\ell\), such that, for every \(n\ge1\),
\(
\gamma_n
\le \bar\gamma_n = 
C_\gamma
n^{\frac{d}{2\nu+d}}
\log^{\frac{2\nu}{2\nu+d}}(e+n).
\)
\end{lemma}

\begin{proof}
Let
\[
s=\nu+\frac d2.
\]
Up to a positive constant depending only on the kernel parameters, the
spectral coefficients of \(k_{\nu,\ell}\) are
\[
w_j=(a_\ell+\lambda_j)^{-s}.
\]

We use two spectral estimates for the Laplace--Beltrami operator. First, the local Weyl law for the spectral function
\cite{hormander1968spectral} implies that there exists
\(C_{\mathrm W}>0\) such that
\[
\sup_{x\in\mathcal X}
\sum_{\lambda_j\le L}\varphi_j(x)^2
\le
C_{\mathrm W}(1+L)^{d/2},
\qquad L\ge0.
\]
Second, Weyl's eigenvalue asymptotics for the Laplace--Beltrami operator
\cite{hormander1968spectral,chavel1984eigenvalues} imply
\(\lambda_j\asymp j^{2/d}\) as \(j\to\infty\). Hence, after adjusting the
constant to include the finitely many low-frequency modes,
\[
1+\lambda_j
\ge
c_{\mathrm W}(1+j)^{2/d}
\]
for some \(c_{\mathrm W}>0\) and all \(j\ge0\).

For \(D\ge1\), decompose
\[
k_{\nu,\ell}=k_D+r_D,
\]
where \(k_D\) contains the first \(D\) spectral modes. We first control the
diagonal of the residual kernel. Let
\[
L_D=a_\ell+\lambda_D.
\]
Partition the indices \(j\ge D\) into
\[
\mathcal J_q
=
\left\{
j\ge D:
2^qL_D
\le
a_\ell+\lambda_j
<
2^{q+1}L_D
\right\},
\qquad q\ge0.
\]
By the local Weyl bound for the spectral function of the
Laplace--Beltrami operator~\cite{hormander1968spectral}, 
\[
\begin{aligned}
r_D(x,x)
&=
\sum_{j\ge D}
w_j\varphi_j(x)^2\\
&\le
C
\sum_{q=0}^{\infty}
(2^qL_D)^{-s}
\sum_{j\in\mathcal J_q}\varphi_j(x)^2\\
&\le
C'
\sum_{q=0}^{\infty}
(2^qL_D)^{-s}
(2^{q+1}L_D)^{d/2}\\
&=
C''
L_D^{-\nu}
\sum_{q=0}^{\infty}2^{-q\nu}.
\end{aligned}
\]
Since \(\nu>0\), the geometric series converges. Weyl eigenvalue growth then
gives
\[
\sup_{x\in\mathcal X}r_D(x,x)
\le
C_r D^{-2\nu/d}
\]
for some \(C_r>0\).

Now fix arbitrary \(x_1,\ldots,x_n\in\mathcal X\), and decompose their Gram
matrix as
\[
K_n=K_{n,D}+R_{n,D}.
\]
Both matrices are positive semidefinite,
\(\operatorname{rank}(K_{n,D})\le D\), and, since
\(k_{\nu,\ell}(x,x)\le1\),
\[
\operatorname{tr}(K_{n,D})
\le
\operatorname{tr}(K_n)
\le n.
\]
Hence, by concavity of the logarithm,
\[
\log\det(I_n+K_{n,D})
\le
D\log\left(1+\frac nD\right).
\]

For the residual term,
\[
\begin{aligned}
\log\det(I_n+K_n)
&=
\log\det(I_n+K_{n,D})\\
&\quad+
\log\det\left(
I_n+
(I_n+K_{n,D})^{-1/2}
R_{n,D}
(I_n+K_{n,D})^{-1/2}
\right)\\
&\le
\log\det(I_n+K_{n,D})
+
\operatorname{tr}(R_{n,D})\\
&\le
D\log\left(1+\frac nD\right)
+
C_r nD^{-2\nu/d}.
\end{aligned}
\]
Therefore,
\[
\gamma_n
\le
\frac12
\left[
D\log\left(1+\frac nD\right)
+
C_r nD^{-2\nu/d}
\right].
\]

Set
\[
p=\frac{d}{2\nu+d},
\qquad
q=\frac{2\nu}{2\nu+d}=1-p,
\]
and, for \(n\ge2\), choose
\[
D
=
\left\lceil
\left(
\frac{n}{\log(e+n)}
\right)^p
\right\rceil.
\]
Then
\[
D\log\left(1+\frac nD\right)
\le
C
n^p\log^q(e+n),
\]
while
\[
nD^{-2\nu/d}
\le
C'
n^p\log^q(e+n).
\]
Thus
\[
\gamma_n
\le
C_\gamma
n^p\log^q(e+n).
\]
The case \(n=1\) is absorbed by increasing \(C_\gamma\), which completes the
proof.
\end{proof}

By Lemma~\ref{lem:manifold-matern-information},
\[
\Gamma_{i,T}
\le
\gamma_{N_i(T)}
\le
C_\gamma
N_i(T)^p
\log^q(e+N_i(T))
\le
C_\gamma
N_i(T)^pL_T^q.
\]
Since \(\sum_{i=1}^KN_i(T)=N_T\) and \(p\in(0,1)\),
concavity gives
\[
\sum_{i=1}^KN_i(T)^p
\le
K^{1-p}N_T^p.
\]
Hence
\[
\Gamma_T^\Sigma
\le
C_\gamma
K^{1-p}N_T^pL_T^q.
\]
Because \(1-p=q=2\nu/(2\nu+d)\),
\[
\Gamma_T^\Sigma
\le
C_\gamma
K^{\frac{2\nu}{2\nu+d}}
N_T^{\frac{d}{2\nu+d}}
L_T^{\frac{2\nu}{2\nu+d}}.
\]

Likewise,
\[
\beta_{i,t-1}
\le
S+
R\sqrt{
2\left[
C_\gamma
N_T^pL_T^q
+
1+
\log\frac{K}{\delta}
\right]
}
=
\bar\beta_T.
\]
Substituting these two bounds into
Theorem~\ref{thm:manifold-regret} proves the result.
\end{proof}

\subsubsection{Proof of Theorem \ref{thm:manifold-heterogeneous-regret}}

\begin{proof}
Let
\[
\Delta_t(i)
=
\frac{V_t(i_t^\star)-V_t(i)}{M_t},
\]
and let \(\mathcal P_t^\star\) denote the exploration matching at round \(t\).
Since \(M_t-e_t\) agents select the leader \(h_t\), while the \(e_t\)
remaining agents are matched one-to-one with challengers,
\[
R_T
=
\sum_{t=1}^T
(M_t-e_t)\Delta_t(h_t)
+
\sum_{t=1}^T R_t^{\mathrm{exp}},
\]
where
\[
R_t^{\mathrm{exp}}
=
\sum_{(m,i)\in\mathcal P_t^\star}
\Delta_t(i).
\]

Let \(\mathcal E\) denote the simultaneous confidence event
\[
\left|
\widehat V_t(i)-V_t(i)
\right|
\le
\beta_{i,t-1}\sigma_{i,t}^{\mathrm{pop}},
\qquad
\forall i\in[K],\ t\in[T].
\]
By the population-functional confidence result,
\(\mathbb P(\mathcal E)\ge1-\delta\). We work on \(\mathcal E\).
Since \(N_i(t-1)\le N_T\) and \(\bar\gamma_n\) is nondecreasing,
\[
\beta_{i,t-1}\le\beta_T
\]
for every \(i\) and \(t\).

Because \(h_t\in\arg\max_i U_t(i)\),
\[
U_t(h_t)
\ge
U_t(i_t^\star)
\ge
V_t(i_t^\star).
\]
On \(\mathcal E\),
\[
U_t(h_t)
=
\widehat V_t(h_t)
+
\beta_{h_t,t-1}\sigma_{h_t,t}^{\mathrm{pop}}
\le
V_t(h_t)
+
2\beta_{h_t,t-1}\sigma_{h_t,t}^{\mathrm{pop}}.
\]
Therefore,
\[
V_t(i_t^\star)-V_t(h_t)
\le
2\beta_T\sigma_{h_t,t}^{\mathrm{pop}},
\]
and hence
\[
(M_t-e_t)\Delta_t(h_t)
\le
2\beta_T
\frac{M_t-e_t}{M_t}
\sigma_{h_t,t}^{\mathrm{pop}}.
\]

Define the embedding of the agents selecting the leader by
\[
\Psi_t^h
=
\sum_{m\in\mathcal Q_{t,h_t}}\phi(x_m).
\]
Since
\(
\Psi_t=\Psi_t^h+(\Psi_t-\Psi_t^h)
\),
the triangle inequality gives
\[
\sigma_{h_t,t}^{\mathrm{pop}}
\le
\|\Psi_t^h\|_{\mathcal A_{h_t,t-1}^{-1}}
+
\|\Psi_t-\Psi_t^h\|_{\mathcal A_{h_t,t-1}^{-1}}.
\]
Thus
\[
\begin{aligned}
(M_t-e_t)\Delta_t(h_t)
\le{}&
2\beta_T
\frac{M_t-e_t}{M_t}
\|\Psi_t^h\|_{\mathcal A_{h_t,t-1}^{-1}}
\\
&+
2\beta_T
\frac{M_t-e_t}{M_t}
\|\Psi_t-\Psi_t^h\|_{\mathcal A_{h_t,t-1}^{-1}}.
\end{aligned}
\]

We first control the uncertainty contributed by the leader batch. Let
\[
A_t=\mathcal A_{h_t,t-1},
\qquad
z_m=A_t^{-1/2}\phi(x_m),
\quad
m\in\mathcal Q_{t,h_t},
\]
and
\[
C_t
=
\sum_{m\in\mathcal Q_{t,h_t}}
z_m\otimes z_m.
\]
Since \(|\mathcal Q_{t,h_t}|=M_t-e_t\),
\[
\begin{aligned}
\|\Psi_t^h\|_{A_t^{-1}}^2
&=
\|
\sum_{m\in\mathcal Q_{t,h_t}}z_m
\|_{\mathcal H_k}^2\\
&\le
(M_t-e_t)
\sum_{m\in\mathcal Q_{t,h_t}}
\|z_m\|_{\mathcal H_k}^2\\
&=
(M_t-e_t)\operatorname{tr}(C_t).
\end{aligned}
\]
Because \(A_t\succeq I\) and \(k(x,x)\le1\),
\[
\operatorname{tr}(C_t)
=
\sum_{m\in\mathcal Q_{t,h_t}}
\langle
\phi(x_m),A_t^{-1}\phi(x_m)
\rangle_{\mathcal H_k}
\le
M_t-e_t.
\]
Hence every eigenvalue \(\lambda\) of \(C_t\) lies in
\([0,M_t-e_t]\). Since \(z/\log(1+z)\) is nondecreasing on
\([0,\infty)\),
\[
\lambda
\le
\frac{M_t-e_t}{\log(1+M_t-e_t)}
\log(1+\lambda).
\]
Summing over the eigenvalues gives
\[
\operatorname{tr}(C_t)
\le
\frac{M_t-e_t}{\log(1+M_t-e_t)}
\log\det(I+C_t).
\]

Define the information-gain increment of the leader batch by
\[
\Delta\Gamma_t^h
=
\frac12\log\det(I+C_t).
\]
By the determinant identity for the batch update of arm \(h_t\), this is
exactly the increase in its realized information gain produced by the
\(M_t-e_t\) leader observations at round \(t\). Therefore,
\[
\|\Psi_t^h\|_{A_t^{-1}}^2
\le
\frac{2(M_t-e_t)^2}
{\log(1+M_t-e_t)}
\Delta\Gamma_t^h,
\]
and consequently
\[
\left(
\frac{M_t-e_t}{M_t}
\|\Psi_t^h\|_{A_t^{-1}}
\right)^2
\le
2
\frac{(M_t-e_t)^4}
{M_t^2\log(1+M_t-e_t)}
\Delta\Gamma_t^h.
\]
Cauchy--Schwarz now yields
\[
\sum_{t=1}^T
\frac{M_t-e_t}{M_t}
\|\Psi_t^h\|_{A_t^{-1}}
\le
\sqrt{
2H_T^{\mathrm{het}}
\sum_{t=1}^T\Delta\Gamma_t^h
}.
\]

For each arm \(i\), the information increments obtained on rounds for which
\(h_t=i\) are a subset of the nonnegative information increments accumulated
by that arm. Hence
\[
\sum_{t=1}^T\Delta\Gamma_t^h
\le
\sum_{i=1}^K\Gamma_{i,T}
=
\Gamma_T^\Sigma.
\]
It follows that
\[
\sum_{t=1}^T
(M_t-e_t)\Delta_t(h_t)
\le
2\beta_T
\sqrt{
2\Gamma_T^\Sigma H_T^{\mathrm{het}}
}
+
2\beta_TD_T,
\]
where, for this proof only, we write
\[
D_T
=
\sum_{t=1}^T
\frac{M_t-e_t}{M_t}
\|\Psi_t-\Psi_t^h\|_{\mathcal A_{h_t,t-1}^{-1}}.
\]
Combining this inequality with the regret decomposition gives
\[
R_T
\le
2\beta_T
\sqrt{
2\Gamma_T^\Sigma H_T^{\mathrm{het}}
}
+
2\beta_TD_T
+
\sum_{t=1}^T R_t^{\mathrm{exp}}.
\]

We next eliminate the two auxiliary terms. Since exactly \(e_t\) agents do
not select the leader,
\[
\Psi_t-\Psi_t^h
=
\sum_{m\in\mathcal M_t\setminus\mathcal Q_{t,h_t}}
\phi(x_m).
\]
Moreover,
\(\mathcal A_{h_t,t-1}^{-1}\preceq I\) and
\(\|\phi(x_m)\|_{\mathcal H_k}^2=k(x_m,x_m)\le1\), so
\[
\|\Psi_t-\Psi_t^h\|_{\mathcal A_{h_t,t-1}^{-1}}
\le
\sum_{m\in\mathcal M_t\setminus\mathcal Q_{t,h_t}}
\|\phi(x_m)\|_{\mathcal A_{h_t,t-1}^{-1}}
\le
e_t.
\]
Therefore,
\[
D_T
\le
\sum_{t=1}^T
\frac{M_t-e_t}{M_t}e_t
\le
N_T^{\mathrm{exp}}.
\]
Also, because \(f_i(x)\in[0,1]\), each population-average reward belongs to
\([0,1]\), and hence
\[
0\le\Delta_t(i)\le1.
\]
Since \(|\mathcal P_t^\star|=e_t\),
\[
R_t^{\mathrm{exp}}
=
\sum_{(m,i)\in\mathcal P_t^\star}\Delta_t(i)
\le
e_t.
\]
Thus,
\[
\sum_{t=1}^T R_t^{\mathrm{exp}}
\le
N_T^{\mathrm{exp}},
\]
which proves
\[
R_T
\le
2\beta_T
\sqrt{
2\Gamma_T^\Sigma H_T^{\mathrm{het}}
}
+
(2\beta_T+1)N_T^{\mathrm{exp}}.
\]

We now prove the second bound. Suppose that every exploratory assignment
\((m,i)\in\mathcal P_t^\star\) satisfies
\[
I_{i,t}(m)
=
\frac12
\log\!\left(
1+
\left\langle
\phi(x_m),
\mathcal A_{i,t-1}^{-1}\phi(x_m)
\right\rangle_{\mathcal H_k}
\right)
\ge\eta.
\]
Because \(\mathcal P_t^\star\) is a matching, each challenger is assigned to
at most one agent at round \(t\). Thus, if \((m,i)\in\mathcal P_t^\star\),
arm \(i\neq h_t\) receives exactly one observation at that round and its
design operator is updated as
\[
\mathcal A_{i,t}
=
\mathcal A_{i,t-1}
+
\phi(x_m)\otimes\phi(x_m).
\]
By the matrix determinant lemma,
\[
\Gamma_{i,t}-\Gamma_{i,t-1}
=
\frac12
\log\!\left(
1+
\left\langle
\phi(x_m),
\mathcal A_{i,t-1}^{-1}\phi(x_m)
\right\rangle_{\mathcal H_k}
\right)
=
I_{i,t}(m).
\]
Hence the total information gain contributed by exploratory assignments is
\[
\sum_{t=1}^T
\sum_{(m,i)\in\mathcal P_t^\star}
I_{i,t}(m).
\]
These are nonnegative information-gain increments contained in the total
increments accumulated across all arms. Therefore,
\[
\sum_{t=1}^T
\sum_{(m,i)\in\mathcal P_t^\star}
I_{i,t}(m)
\le
\Gamma_T^\Sigma.
\]
On the other hand, the threshold condition gives
\[
\sum_{t=1}^T
\sum_{(m,i)\in\mathcal P_t^\star}
I_{i,t}(m)
\ge
\eta
\sum_{t=1}^T|\mathcal P_t^\star|
=
\eta N_T^{\mathrm{exp}}.
\]
Consequently,
\[
N_T^{\mathrm{exp}}
\le
\frac{\Gamma_T^\Sigma}{\eta}.
\]
Substituting this inequality into the first regret bound yields
\[
R_T
\le
2\beta_T
\sqrt{
2\Gamma_T^\Sigma H_T^{\mathrm{het}}
}
+
\frac{2\beta_T+1}{\eta}
\Gamma_T^\Sigma.
\]

Finally, arm \(i\) receives \(N_i(T)\) observations in total, so
\[
\Gamma_{i,T}
\le
\gamma_{N_i(T)}
\le
\bar\gamma_{N_i(T)}
\le
\bar\gamma_{N_T},
\]
where the final inequality uses
\(N_i(T)\le N_T\) and the monotonicity of \(\bar\gamma_n\). Summing over
\(i\in[K]\) gives
\[
\Gamma_T^\Sigma
\le
K\bar\gamma_{N_T}.
\]
This completes the proof.
\end{proof}

\end{document}